\documentclass{article}

\usepackage{PRIMEarxiv}
\usepackage{dsfont}
\usepackage{natbib}
\usepackage{amssymb}
\usepackage{amsmath,amsthm,amssymb,amsfonts}
\usepackage{algorithm}
\usepackage{algpseudocode}
\usepackage{booktabs}
\usepackage{graphicx}

\newtheorem{theorem}{Theorem}[section]
\newtheorem{lemma}[theorem]{Lemma}

\newtheorem{proposition}{Proposition}

\usepackage{xcolor}

\def\Alt{\text{Alt}}

\newcounter{alphasect}
\def\alphainsection{0}

\let\oldsection=\section
\def\section{%
  \ifnum\alphainsection=1%
    \addtocounter{alphasect}{1}
  \fi%
\oldsection}%

\renewcommand\thesection{%
  \ifnum\alphainsection=1%
    \Alph{alphasect}
  \else%
    \arabic{section}
  \fi%
}%

\newenvironment{alphasection}{%
  \ifnum\alphainsection=1%
    \errhelp={Let other blocks end at the beginning of the next block.}
    \errmessage{Nested Alpha section not allowed}
  \fi%
  \setcounter{alphasect}{0}
  \def\alphainsection{1}
}{%
  \setcounter{alphasect}{0}
  \def\alphainsection{0}
}%

\def\exp{\mathrm{exp}}

\usepackage{natbib}
\usepackage{amsmath}
\usepackage{amsthm}
\usepackage{amssymb}
\usepackage{algorithm}

\usepackage{amsmath}
\usepackage{amssymb}
\usepackage{mathtools}
\usepackage{amsthm}
\usepackage{accents}
\usepackage{bbm}
\usepackage{multicol,lipsum,xparse}

\usepackage[utf8]{inputenc} 
\usepackage[T1]{fontenc}    
\usepackage[hidelinks]{hyperref}
\usepackage{url}            
\usepackage{booktabs}       
\usepackage{amsfonts}       
\usepackage{nicefrac}       
\usepackage{microtype}      
\usepackage{lipsum}
\usepackage{fancyhdr}       
\usepackage{graphicx}       
\graphicspath{{media/}}     

\theoremstyle{plain}

\title{Near Optimal Pure Exploration in Logistic Bandits}

\author{
  Eduardo Ochoa Rivera, Ambuj Tewari \\
  University of Michigan, \\
  Ann Arbor, USA \\
  \today
}

\begin{document}
\maketitle

\begin{abstract}
Bandit algorithms have garnered significant attention due to their practical applications in real-world scenarios. However, beyond simple settings such as multi-arm or linear bandits, optimal algorithms remain scarce. Notably, no optimal solution exists for pure exploration problems in the context of generalized linear model (GLM) bandits. In this paper, we narrow this gap and develop the first track-and-stop algorithm for general pure exploration problems under the logistic bandit called logistic track-and-stop (Log-TS). Log-TS is an efficient algorithm that asymptotically matches an approximation for the instance-specific lower bound of the expected sample complexity up to a logarithmic factor.
\end{abstract}


\section{Introduction}\label{section:intro}

The multi-arm bandit (MAB) problem is one of the most important and classical problems in sequential decision-making under uncertainty, and it has been studied for nearly a century \citep{thompson1933likelihood, robbins1952bandit}. The most common setting involves regret minimization, where the goal is to minimize regret over a finite time horizon. This scenario has been extensively studied under linear reward functions \citep{abe1999associative, dani2008stochastic, rusmevichientong2010linearly}. Extensions of the linear case include generalized linear models (GLM) \citep{filippi2010parametric, faury2020improved}, Lipschitz bandits \citep{bubeck2012regret} and spectral bandits \citep{valko2014spectral}. 

On the other hand, the pure exploration setting has gained significant attention in recent years, particularly in the context of best arm identification (BAI), which is well understood within the MAB framework \citep{kaufmann2016bai} and in stochastic linear bandits \citep{soare2015bailinear}. Other pure exploration problems have also been studied, such as the thresholding bandit problem (TBP) \citep{locatelli2016optimal, kano2019good}, and top-m arm identification \citep{bubeck13topm, kalyanakrishnan2012pac}. Similarly to the regret minimization problem, there are extensions to the GLM case for BAI \citep{kazerouni2021, jun2021improved}. However, these algorithms use loose inequalities or require warm-up phases that can be prohibitive in practise and they only can be applied to BAI. In this work we narrow this gap with the following contributions:




\begin{itemize}
    \item We propose Log-TS, the first track-and-stop type algorithm for {\em general} pure exploration problems in the logistic bandit setup.
    \item We prove both in-expectation and almost sure upper bounds for the sample complexity of Log-TS.
    \item We provide a lower bound for the expected sample complexity of {\em general}  pure exploration problems (including BAI, top-m, and thresholding bandits) and a tractable approximation. Log-TS matches this lower bound asymptotically up to a logarithmic factor.
     \item We confirm the practical performance of Log-TS for the classical hard instance for pure exploration problems and when the number of arms increases.
\end{itemize}

\subsection{Paper Structure}\label{section:struc}

The reminder of this paper is organized as follows:

In Section \ref{section:rel_work} we discuss the previous work on GLM bandits and pure exploration problems. In Section \ref{section:background} we formulate pure exploration problems under logistic bandits. We also state the definitions and results needed for the construction of the algorithm. In Section \ref{section:lbsc} we present an instance-specific lower bound for the expected sample complexity of {\em general}  pure exploration problems and a tractable approximation. In Section \ref{section:algorithm} we define the components needed for Log-TS: a stopping rule and a sampling rule. In Section \ref{section:sample_complexity} we state the asymptotic upper bounds for the sample complexity of Log-TS. In Section \ref{section:experiments} we describe the numerical experiments for two specific pure exploration problems: BAI and TBP. Finally, in Section \ref{section:discussion} we discuss the presented results and point out future directions.

\section{Related Work}\label{section:rel_work}

Extensions of stochastic linear bandits have attracted attention due to the restrictive assumption of linearity in real-world applications. For instance, when observing binary rewards, modeling the mean reward as a linear function can be inaccurate. A natural extension in such cases is generalized linear model (GLM) bandits, particularly logistic bandits. GLM bandits were studied by \citet{filippi2010parametric}, where they used a tail inequality similar to the one in the linear case \citep{rusmevichientong2010linearly}, combined with the worst-case behavior of the non-linearity of the link function $\mu$, $\kappa=\sup _{x \in \mathcal{X}, \theta \in \Theta} 1/\dot{\mu}\left(x^{\top} \theta\right)$. More precisely, $|\mu(X^{\top}_t\theta_t^{(1)})-\mu(X^{\top}_t\theta)| \le \rho(t, \delta)$ with probability $1-\delta$ where $\rho(t, \delta) = \mathcal{O} (\kappa\sqrt{d\log (t)\log (d/\delta)})$, $\theta_t^{(1)}$ is the projection of the maximum likelihood estimator (MLE) estimator and $X_t$ is any random varible in $\mathcal{X}$. 

As pointed out in \citet{faury2020improved}, $\kappa$ can be restrictively large for certain link functions in real applications, such as in the case of logistic bandits. It can be shown that $\kappa \geq \exp \left(\max _{x \in \mathcal{X}}\left|x^{\top} \theta_*\right|\right)$ for the logistic model. To address this, \citet{faury2020improved} proposed a new tail inequality that takes into account the local curvature of the link function. This inequality is independent of $\kappa$, significantly improving the regret upper bound. Another key property exploited in their work is the self-concordance of the logistic loss, which helps to bound the prediction errors $ |\mu(x^{\top} \theta_*)-\mu(x^{\top} \hat{\theta}_t)|$ using their tail inequality. 

A common assumption in GLM bandits is that $\|\theta^*\|<S$. While this assumption helps control errors, it also adds complexity, as some algorithms require projecting the maximum likelihood estimator (MLE) onto the parameter space $\Theta = \{\theta \in \mathbb{R}^d: \|\theta\| < S\}$, which can be computationally intensive. In \citet{russac2021self}, the authors exploit the self-concordance property more effectively, eliminating the need for this projection step. 

Another challenge in GLM bandits is that the MLE and the Fisher information matrix 
cannot be updated recursively, which increases the number of operations per round. This issue is addressed by \citet{faury2022jointly}, who propose an online procedure with a warm-up phase.

Track and stop (TS) algorithms have been among the most common and successful approaches in pure exploration problems due to their asymptotic optimality \citep{kaufmann2016bai,jedra2020optbailinear, degenne2020pureexp, wang2021fast}. The core of these algorithms is to track the oracle proportions of arm draws defined by the sample complexity lower bound. Then, stopping rules are designed using the generalized likelihood ratio. While this approach often involves complex optimization, \citet{jedra2020optbailinear} demonstrated that it is possible to achieve asymptotic optimality even when the optimal weights are not updated at every time, a condition known as the lazy setting.

However, no version of the TS algorithm exists for GLM bandits due to the complexity of the lower bound. The first pure exploration work in GLM bandits that we are aware of is by \citet{kazerouni2021}. In their study, they used a loose inequality from \citet{filippi2010parametric}, but as other authors have noted, relying on this worst-case inequality can lead to a dependency on an exponential factor. A more recent approach was proposed by \citet{jun2021improved}, who used a sharper inequality from \citet{faury2020improved}, accounting for the curvature of the logistic function. They designed an algorithm inspired by RAGE \citep{fiez2019sequential} which is an algorithm for BAI in the linear bandit case and provided a high-probability upper bound for the sample complexity. Additionally, they derived an instance-specific sample complexity lower bound. However, their algorithm requires a warm-up phase that depends on $\kappa_0 =\sup _{x \in \mathcal{X}} 1/\dot{\mu}\left(x^{\top} \theta^*\right)$ and the number of arms, which can be restrictive in real-world applications.

\section{Background and Preliminaries}\label{section:background}

\paragraph{Notation} For any vector $x \in \mathbb{R}^d$ and any positive definite matrix $\mathbf{M} \in \mathbb{R}^{d \times d}$, we define $\|x\|_{\mathbf{M}}:=\sqrt{x^{\top} \mathbf{M} x}$ as the $\ell^2$-norm of $x$ weighted by $\mathbf{M}$. When $\mathbf{M} = \mathbf{I}_{d \times d}$ is the identity matrix, we simply write $\|x\|:=\|x\|_{\mathbf{M}}$. We define $\lambda_{\min }(\mathbf{M})$ and $\lambda_{\max }(\mathbf{M})$  the smallest and largest eigenvalue of $\mathbf{M}$ respectively. We also denote by $\operatorname{Tr}(\mathbf{M})$ the trace of the matrix. For two matrices $\mathbf{A}$ and $\mathbf{B}$, $\mathbf{A} \succ \mathbf{B}$ means that $\mathbf{A}-\mathbf{B}$ is positive definite. We define $\mathcal{B}(d):=\left\{x \in \mathbb{R}^d:\|x\| \leq 1\right\}$ the $d$-dimensional ball of radius 1 under the norm $\ell^2$. For an univariate function $f$ we define $\dot{f}$ its derivative. We define $\Sigma := \{ w \in [0,1]^K : \sum_{k} w_k = 1 \}$ the $K-1$ simplex. For any $w, w' \in \Sigma$, we define $d_{\infty}(w,w') = \max _{k\in[K]} |w_k - w_k'|$, and for any compact set $C \subseteq \Sigma$, $d_{\infty}(w,C) = \min _{w' \in C} d_{\infty}(w,w')$. Finally, for $w \in \Sigma$, we define $\operatorname{supp}(w) =\{i \in [K]: w_i > 0\} $

\subsection{Settings}\label{subsection:settings}

Let $\mathcal{X} \subseteq \mathbb{R}^d$ be a finite set of arms, where $|\mathcal{X}| = K$ and a unknown parameter $\theta^* \in \Theta$. We consider the stochastic logistic bandit, where at each round $t \ge 1$ the decision maker selects an arm $x_t \in \mathcal{X}$ according to a sampling rule based on previously observed samples and obtains a reward $r_t \sim \text{Bernoulli}(\mu(x_t^{\top}\theta^{*}))$, where $\mu(x) = \frac{1}{1+e^{-x}}$. It then proceeds to the next round.

We are interested in general pure exploration problems, where the goal is to identify the true answer $i^{\star}(\theta^*)$ that belongs to a finite set $\mathcal{I}$ of possible answers (e.g., for best arm identification $i^{\star}(\theta^*) = \arg \max _{x\in\mathcal{X}} \mu(x^{\top}\theta^{*})$). We will assume $i^{\star}(\theta^*)$ is unique. We consider the fixed confidence setting, where the objective is to accurately identify $i^{\star}(\theta^*)$ with high probability as soon as possible. Formally, the sampling rule defines for all $t \ge 1$ a function $\pi_t$ from $(\mathcal{X} \times \{0,1\})^{t-1}$  to the space of probability distributions on $\mathcal{X}$, which is measurable with respect to the $\sigma$-algebra $\mathcal{F}_t := \sigma(\{x_s, r_s\}_{s \le t})$. We call that $\sigma$-algebra history before time $t$. At time $\tau$, where $\tau$ is an stopping time with respect to $\mathcal{F}_t$, and given an estimator $\hat{\theta}_{\tau}$ of $\theta^*$ the algorithm stops and the decision $i^{\star}(\hat{\theta}_{\tau})$ is made. We say the algorithm is $\delta$-correct if $\mathbb{P}_{\theta}[\tau_{\delta}<\infty, i^{\star}(\hat{\theta}_{\tau}) \ne i^{\star}(\theta)] < \delta$. Then, the goal is to design a $\delta$-correct algorithm that minimize the expected sample complexity $\mathbb{E}_{\theta}[\tau_{\delta}]$.

\paragraph{Assumptions} We will make the usual assumption in the logistic bandit problem

\begin{itemize}
    \item $\|x\| \in \mathcal{B}(d)$ for all $x\in\mathcal{X}$
    \item  $\Theta = \{\theta \in \mathbb{R}^d: \|\theta\| \le S \}$, $S > 0$
    \item $\mathcal{X}$ spans $\mathbb{R}^d$
    \item We have access to $\kappa_0 :=\sup _{x \in \mathcal{X}} 1/\dot{\mu}\left(x^{\top} \theta^*\right)$ 
\end{itemize}

\subsection{Maximum Likelihood Estimator}\label{subsection:mle}

For logistic regression setting, we can estimate the parameter $\theta^*$ using the MLE. At time $t$, the log-likelihood can be expressed as: 
$$
\begin{aligned}\label{eq:loglike}
    \mathcal{L}_t(\theta)&= \sum_{s=1}^t r_s \log \mu\left(x_s^{\top} \theta\right)\\
    &+\left(1-r_s\right) \log \left(1-\mu\left(x_s^{\top} \theta\right)\right)
\end{aligned}
$$
and the MLE is given by $\hat{\theta}_t = \underset{\theta \in \mathbb{R}^d}{\arg \max \mathcal{L}_t(\theta)}$. We also define Fisher information or Hessian matrix at $\theta$ as
\begin{align}
    \mathbf{H}_t(\theta)=\sum_{s=1}^t \dot{\mu}\left(x_s^{\top} \theta\right) x_s x_s^{\top} ,
\end{align}
and the design matrix as
\begin{align}
    \mathbf{A}_t = \sum_{s=1}^t x_s x_s^{\top}.
\end{align}
Similarly, for $w\in\Sigma$, we define 
$$\mathbf{H}_{w}(\theta)=\sum_{x \in \mathcal{X}} w_x \dot{\mu} (x^{\top} \theta) x x^{\top}.$$
Note that if $w_x = N_t(x)/t$, where $N_t(x)$ is the number of times the arm $x$ has been selected, then $t\mathbf{H}_{w}(\theta) = \mathbf{H}_t(\theta)$. We also denote $g_t(\theta) = \sum_{s=1}^{t} \mu\left(x_s^{\top} \theta\right) x_s$. This function plays an important role in the concentration inequalities as shown in \citet{faury2020improved}. In particular, we have that $g(\hat{\theta}_t) = \sum_{s=1}^{t} r_s x_s$ by definition of the MLE.

\subsection{Concentration}\label{subsection:consentration}

We will use the recent concentration tools developed by \citet{faury2020improved} for the logistic bandit. In particular, for all $t \ge 1$, $\|g_t(\hat{\theta}_t)-g_t\left(\theta_*\right)\|_{\mathbf{H}_t(\theta^*)^{-1}} \le \gamma_t(\delta)$ with probability at least $1-\delta$ for some function $\gamma_t(\delta)$. For completeness, we prove a slightly modified version of this inequality because we need to use the unregularized MLE to guarantee convergence to the true parameter. Addionally, we also use the self-concordant property of the logistic regression so we can apply the inequality to guarantee the algorithm is $\delta$-correct as explained in Section \ref{subsection:stopping_rule}.

\begin{lemma}\label{lemma:concentration}
    Let $\delta \in(0,1]$ and $\lambda(t) > 0$ for $ t \ge 1$. If exist $t_0 \ge 1$ such that for $t \ge t_0$, $\lambda_{\min }(\mathbf{H}_t\left(\theta_*\right)) > \lambda(t)$, with probability at least $1-\delta$:

    \begin{align}\label{eq:cont_inequality_gamma}
        \forall t \geq t_0, \quad\left\|g_t(\hat{\theta}_t)-g_t\left(\theta_*\right)\right\|_{\mathbf{H}_t^{-1}\left(\theta_*\right)} \leq \gamma_t(\delta)    \ ,
    \end{align}

    where $\gamma_t(\delta):=\frac{\sqrt{\lambda(t)}}{2}+\frac{4}{\sqrt{\lambda(t)}} \log \left(\frac{2^d}{\delta}\left(\frac{L t}{\lambda (t)d}\right)^{\frac{d}{2}}\right)$
\end{lemma}

The main difference compare to the original inequality is the assumption $\lambda_{\min }(\mathbf{H}_t\left(\theta_*\right)) > \lambda (t)$. In our case this will be guaranteed by the forced exploration component of sampling rule and the knowledge of $\kappa_0$. The forced exploration can be thought as an adaptive warm-up phase. Moreover, the amount of forced exploration needed to guarantee $\delta$-correctness of the algorithm will depend on $\lambda (t)$. As \citet{faury2020improved} pointed out, we can use the bound in \citet{abbasi2011improved} to derive another high-probability bound

\begin{align}\label{eq:cont_inequality_linear}
\left\|g_t(\hat{\theta}_t)-g_t\left(\theta_*\right)\right\|_{\mathbf{H}_t^{-1}} = \mathcal{O}(\sqrt{ \kappa d \log \left(t/\delta\right)}) \ .
\end{align}

Although the bound in Eq. \eqref{eq:cont_inequality_gamma} is independent of $\kappa$, it has a disadvantage compare to Eq. \eqref{eq:cont_inequality_linear}. It has an extra factor of $\sqrt{ d \log \left(t/\delta\right)}$. Unfortunately, this impacts the asymptotic behavior of the algorithm, the upper bound will be proportional to $(\log (1/\delta))^2$ instead of $\log (1/\delta)$ for small values of $\delta$. However, in practice for a fixed $\delta$, if we set $\lambda (t) = \mathcal{O}(\log(t))$ we observe that $\gamma_t (\delta) = \mathcal{O}(\sqrt{d \log \left(t\right)})$. 


\paragraph{Projection step}

We introduce the projection of the MLE estimator $\hat{\theta}_t$ onto $\Theta$ as

\begin{align}
    \theta_t^{(1)}=\underset{\theta \in \Theta}{\arg \min }\left\|g_t(\theta)-g_t(\hat{\theta}_t)\right\|_{\mathbf{H}_t^{-1}(\theta)} .
\end{align}

    
Thanks to the fact that $\theta^{(1)}_t \in  \Theta$, we can establish an upper bound on bound $
\|\theta^{(1)}_t-\theta^*\|_{\mathbf{H}_t(\theta^{(1)}_t)}$, leveraging the self-concordance property of the logistic loss (See Lemma \ref{lemma:selfconcordance}). Note that $\hat{\theta}_t = \theta_t^{(1)}$ when $\|\hat{\theta}_t\| \le S$.

\section{Sample Complexity Lower Bound}\label{section:lbsc}

In this section, we provide an instance-specific lower bound for the expected sample complexity in \textit{general} pure exploration problems within the logistic bandits setting. \citet{jun2021improved} presented a similar lower bound only for BAI under logistic bandits. In contrast, our bound applies to a broader class of pure exploration problems. Moreover, we derive an approximation of the lower bound using the Taylor expansion of the KL divergence, making it tractable for some pure exploration problems such as BAI, TBP and top-m best arm identification (See Section \ref{section:experiments} and \ref{section:examples} in the appendix). This approximation provides an optimal proportion of arm draws, which we will track in the sampling rule of Log-TS similarly to previous track-and-stop algorithms \citep{kaufmann2016bai, jedra2020optbailinear}.


\paragraph{Alternative} For any $\theta \in \Theta$ we define the alternative to $i^{\star}(\theta)$, denoted by $\Alt(\theta)$, as the set of parameters where the answer $i^{\star}(\theta)$ is not correct. Formally, $\operatorname{Alt}(\theta):=\left\{\lambda \in \Theta: i^{\star}(\lambda) \neq i^{\star}(\theta)\right\}$.

\begin{theorem}\label{thm:thmsclb}
For any logistic bandit environment $(\mathcal{X}, \theta)$ and $\delta > 0$, the sample complexity $\tau_{\delta}$ of any $\delta$-correct strategy satisfies:
$$
\mathbb{E}_{\theta}[\tau_{\delta}] \geq \log (1 / 2.4 \delta) \frac{1}{T^{\star}(\theta)^{-1} + C(\theta)},
$$

Where $T^{\star}(\theta)^{-1}:=\max _{w \in \Sigma} \inf _{\lambda \in \operatorname{Alt}(\theta)} \frac{1}{2}\|\theta-\lambda\|_{\mathbf{H}_w(\theta)}^2$ and $C(\theta)$ is an instance-specific constant measuring the precision of the quadratic approximation to the \text{KL} divergence.

\end{theorem}

\begin{proof}
Let $\lambda \in \Alt(\theta)$, 
we can apply a slightly modified version of Theorem 33.5 from \citet{lattimore2020bandit} which has origin in the transportation theorem \citep{kaufmann2016bai} to show.
\begin{align}\label{eq:trans_lemma}
  \sum_{x \in \mathcal{X}} \mathbb{E}_{\theta}\left[N_{x}\right] \text{KL}_x\left( \theta, \lambda\right) \geq \log (1 / 2.4 \delta)  
\end{align}


Now, for each $x \in \mathcal{X}$ we can approximate the KL divergence with the second order Taylor expansion
$$
\begin{aligned}\label{eq:kl_approx}
\text{KL}_x\left( \theta,\lambda\right) &\approx \text{KL}_x\left( \theta,\theta\right) + (\theta - \lambda)^{\top}
\nabla \left . \text{KL}_x\left(\theta, \cdot\right)\right|_{\theta}\\
& + \frac{1}{2}(\theta - \lambda)^{\top} \left .\textbf{H}_{\text{KL}}\left(\theta,\cdot\right)\right|_{\theta} (\theta - \lambda) + R_x(\lambda)\\
&\approx \frac{1}{2} (\theta - \lambda)^{\top} \dot{\mu} (x^{\top} \theta) x x^{\top} (\theta - \lambda) + R_x(\lambda)
\end{aligned}
$$

Where $\text{KL}_x\left( \theta,\lambda\right)$ stands for the KL divergence between $r | x, \theta$ and $r | x, \lambda$,  $\nabla \left . \text{KL}_x\left(\theta, \cdot\right)\right|_{\theta}$ and $\left .\textbf{H}_{\text{KL}}\left(\theta,\cdot\right)\right|_{\theta}$ are the gradient and Hessian matrix of $\text{KL}_x\left( \theta,\lambda\right)$ with respect to $\lambda$ evaluated in $\theta$. We use the fact that the KL divergence between two distributions from the same exponential family can be expressed as $\text{KL}(\eta_1, \eta_2) = (\eta_1 - \eta_2) \mu_1 + A(\eta_1) - A(\eta_2)$. Where their probability distribution is given by $
p(x \mid \eta)=h(x) \exp \left\{\eta^T T(x)-A(\eta)\right\}
$. In particular, for two GLM models with parameters $\theta$ and $\lambda$ and link function $\mu(\theta)$, we have $\text{KL}_x(\theta, \lambda) = x^{\top}(\theta - \lambda) \mu (x^{\top}\theta) + A(x^{\top}\theta) - A(x^{\top}\lambda)$, $\nabla \text{KL}_x\left(\theta, \lambda\right) = x^{\top} (\mu (x^{\top}\lambda) -\mu (x^{\top}\theta))$ and $\textbf{H}_{\text{KL}}\left(\theta,\lambda\right) = \dot{\mu}(x^{\top}\lambda) xx^{\top}$ . Then, after substituting the KL approximation  in Eq. \eqref{eq:trans_lemma} we obtain
$$ \mathbb{E}_{\theta}[\tau_{\delta}] \left(\frac{1}{2}\|\theta - \lambda\|^2_{\mathbf{H}_{\tau_{\delta}}(\theta)} + \max _{x \in \mathcal{X}} |R_x(\lambda)|\right)  \geq \log (1 / 2.4 \delta)$$
Where $\mathbf{H}_{\tau_{\delta}}(\theta) = \left(
\sum_{x \in \mathcal{X}} \frac{\mathbb{E}_{\theta}\left[N_{x}\right] }{\mathbb{E}_{\theta}[\tau_{\delta}]} \dot{\mu} (x^{\top} \theta) x x^{\top} \right)$, then
$$
\begin{aligned}
     \mathbb{E}_{\theta}[\tau_{\delta}] \sup _{w \in \Sigma}\inf _{\lambda \in \operatorname{Alt}(\theta) }&\left(\frac{1}{2}\|\theta - \lambda\|^2_{\mathbf{H}_w(\theta)}\right. \\
     &\left. + \max _{x \in \mathcal{X}} |R_x(\lambda)|\right)  \geq \log (1 / 2.4 \delta)
\end{aligned}
$$
$$
\begin{aligned}
    \mathbb{E}_{\theta}[\tau_{\delta}] & \left( \sup _{w \in \Sigma} \inf _{\lambda \in \operatorname{Alt}(\theta)} \frac{1}{2}\|\theta-\lambda\|_{\mathbf{H}_w(\theta)}^2 \right . \\
    & \left . + \inf _{\lambda \in \operatorname{Alt}(\theta) } \max _{x \in \mathcal{X}} |R_x(\lambda)| \right)  \geq \log (1 / 2.4 \delta)    
\end{aligned}
$$

If we denote $C(\theta) = \inf _{\lambda \in \operatorname{Alt}(\theta) } \max _{x \in \mathcal{X}} |R_x(\lambda)|$ we conclude the proof.

\end{proof}

As noted previously, \citet{jun2021improved} showed an alternative lower bound for BAI under logistic bandits. 
$$
\mathbb{E}_{\theta}[\tau_{\delta}] \geq \log (1 / 2.4 \delta) \frac{1}{\max _{w \in \Sigma} \inf _{\lambda \in \operatorname{Alt}(\theta)}\left\|\theta-\lambda\right\|_{\mathbf{K}_w(\theta, \lambda)}^2},
$$

Where 
 $\mathbf{K}_w(\theta_1, \theta_2)=\sum_{x \in \mathcal{X}} w_x \beta\left(\theta_1, \theta_2\right) x x^{\top}$ and $\beta(a, b)=\int_0^1(1-t) \dot{\mu}(a+t(b-a)) d t$. The main difference between our bounds is that ours uses a quadratic approximation for the KL divergence, allowing us to establish a direct relationship with the Fisher information matrix $\mathbf{H}$. If the approximation of the KL divergence is accurate for $(\mathcal{X},\theta)$, we expect that $T^{\star}(\theta)^{-1} + C(\theta)\approx \max _{w \in \Sigma} \inf _{\lambda \in \operatorname{Alt}(\theta)}\left\|\theta-\lambda\right\|_{\mathbf{K}_w(\theta, \lambda)}^2$ and then

$$
\begin{aligned}
    \max _{w \in \Sigma} \inf _{\lambda \in \operatorname{Alt}(\theta)} \frac{1}{2} \|\theta-\lambda\|_{\mathbf{H}_w}^2 \le \max _{w \in \Sigma} \inf _{\lambda \in \operatorname{Alt}(\theta)}\left\|\theta-\lambda\right\|_{\mathbf{K}_w}^2
\end{aligned}
$$

where $\mathbf{H}_w = \mathbf{H}_w (\theta)$ and $\mathbf{K}_w = \mathbf{K}_w (\theta, \lambda)$. When the constant $C(\theta)$ is negligible, our lower bound closely resembles the lower bound for the linear case \citep{soare2014bailinear,degenne2020pureexp}, as the constant $C(\theta) = 0$ in this case. Our goal is to design an algorithm with the following asymptotic sample complexity:

$$
\limsup _{\delta \rightarrow 0} \frac{\mathbb{E}_{\theta}[\tau]}{(\log (1 / \delta))^2} \leq T^{\star}(\theta),
$$

\paragraph{Remark.} We may ask when the assumption $C(\theta) \approx 0$ is reasonable. From the proof of Theorem \ref{thm:thmsclb} we have $C(\theta) = \inf _{\lambda \in \operatorname{Alt}(\theta) } \max _{x \in \mathcal{X}} |R_x(\lambda)|$. From this definition, we can see that if there exists a parameter $\lambda \in \Alt (\theta)$ that allows a good approximation of the KL divergence, then $C(\theta)$ will be close to 0. This occurs when there is a $\lambda \in \Alt(\theta)$ sufficiently close to $\theta$, which implies that $T^{\star}(\theta)^{-1}$ will be small as well.  

\section{Logistic Track-and-Stop Algorithm}\label{section:algorithm}

In this section, we present the first track-and-stop type  algorithm for general pure exploration under logistic bandits. First, we propose a modified version of the classical Chernoff stopping rule \citep{kaufmann2016bai}, which utilizes the approximation of the generalized likelihood ratio from Theorem \ref{thm:thmsclb}. We prove that this stopping rule provides a $\delta$-correct algorithm under any sampling rule. 

Next, we define the standard tracking rule with forced exploration \citep{jedra2020optbailinear}. This component of the algorithm tracks the estimated optimal proportion of arm pulls based on the projection of the MLE onto $\Theta$. Thanks to forced exploration, we can ensure that the MLE, and consequently its projection, converge almost surely to $\theta^*$. As a result, the estimated optimal proportion of arm draws converges to a true optimal proportion due to the continuity of the function defined by the optimization problem.

\subsection{Stopping rule}\label{subsection:stopping_rule}

We will use the approximation of the generalized likelihood ratio given in Theorem \ref{thm:thmsclb}. We define the stopping rule as
\begin{align}
    \tau_{\delta} &= \inf \left\{ t \geq 1 : Z(t) > \beta(\delta, t), t \in B \right\}
\end{align}
Where $B =  \{t \ge 1: \lambda_{\min }(\mathbf{A}_s) > \kappa_0 \lambda(s), \forall s \ge t\} $, $Z(t) = \inf _{\lambda \in \Alt(\theta^{(1)}_t)} \frac{1}{2}\|\theta^{(1)}_t-\lambda\|_{H_t(\theta^{(1)}_t)}^2$ and $\beta(\delta, t) = 2((1+2S)\gamma_t(\delta))^2$. Then, by using this stopping rule, we obtain a $\delta$-correct algorithm.

\begin{lemma}\label{lemma:delta_correct}
    Under any sampling rule, we have
    $$\mathbb{P}_\theta\left(\tau_\delta<\infty \wedge i^{\star}(\theta^{(1)}_{\tau_\delta}) \neq i^{\star}(\theta)\right) \leq \delta$$
\end{lemma}

Although Lemma \ref{lemma:concentration} requires $\lambda(\mathbf{H}_t\left(\theta_*\right)) > \lambda (t)$, it is not possible to guarantee this directly since we do not have access to the true matrix $\mathbf{H}_t\left(\theta_*\right)$. Instead, it is enough to ask $t \in B$, thanks to the fact that $\lambda_{\min}(\mathbf{H}_t\left(\theta_*\right)) > \frac{1}{\kappa_0} \lambda_{\min}(\mathbf{A}_t)$. We will see that the condition $t \in B$ is easily met due to forced exploration. In fact, from Lemma \ref{lemma:forced_expl} it is sufficient that $ c_{\mathcal{X}_0}\sqrt{t} > \kappa_0 \lambda (t)$ for some constant $c_{\mathcal{X}_0}$. Although this condition may seem easy to satisfy, it will depend on the relationship between $c_{\mathcal{X}_0}$ and $\kappa_0$. If the problem is highly complex or the set $\mathcal{X}_0$ is flat in some direction, this condition can be restrictive. 

In practice, instead of checking $t \in B$, we can verify if $\lambda_{\min }(\mathbf{H}_t(\hat{\theta})) > \lambda(t)$, which can result in an earlier stopping rule. Another parameter we can adjust is the function $\lambda (t)$. It is sufficient that exists $t^*\ge1$ such that $\lambda (t)$ satisfies $c_{\mathcal{X}_0}\sqrt{t} > \kappa_0 \lambda(t)$ for all $t \geq t^*$. This introduces a trade-off between achieving $Z(t) > \beta(\delta, t)$ and satisfying $c_{\mathcal{X}_0}\sqrt{t} > \kappa_0 \lambda(t)$, reflecting the amount of forced exploration required. 

One technical challenge in using the concentration inequality from Section \ref{subsection:consentration} is that $\mathbf{H}_t(\theta^*)$ requires knowledge of the true parameter. We address this issue by leveraging the generalized self-concordance property of the logistic loss, as noted by \citet{faury2020improved}. This property allows us to control the distance between $\theta^{(1)}_t$ and $\theta^*$
\begin{align}
    \|\theta^{(1)}_t-\theta^*\|_{\mathbf{H}_t^{(1)}} \le 2(1+2S) \|g_t(\hat{\theta}_t)-g_t(\theta^*)\|_{(\mathbf{H}^*_t)^{-1}}
\end{align}
where $\mathbf{H}_t^{(1)} = \mathbf{H}_t(\theta^{(1)}_{t})$ and $\mathbf{H}^*_t= \mathbf{H}_t(\theta^*)$. In \citet{jun2021improved}, the authors also used the self-concordance property to control $\mathbf{H}_t(\hat{\theta}_t)$ using a warm-up phase to control the linear prediction errors $\max _{s \in[t]}|x_s^{\top}(\hat{\theta}_t-\theta^*)| \leq 1$ with the advantage that they do not need the assumption $\|\theta^*\|<S$. 









\subsection{Sampling rule}\label{subsection:tracking}

\paragraph{Forced exploration}

Forced exploration is a crucial component of the Log-TS algorithm, as it allows us to apply Lemma \ref{lemma:concentration} and guarantees the convergence of the MLE estimator. A sampling rule from the family defined in Lemma \ref{lemma:forced_expl} is forced to explore an arm in $\mathcal{X}_0$ (in a round robin manner) if $\lambda_{\min }\left(\sum_{s=1}^t x_s x_s^{\top}\right)$ is too small \citep{jedra2020optbailinear}.

\begin{lemma}\label{lemma:forced_expl}
(Lemma 5 \citep{jedra2020optbailinear})
    Let $\mathcal{X}_0 =\{x_0(1), \dots, x_0(d)\} : \lambda_{\text{min}}(\sum_{x \in \mathcal{X}_0} xx^{\top}) > 0.$ Let $\{b_t\}_{t \geq 0} $ be an arbitrary sequence of arms. Furthermore, define for all $t \geq 1$, $f(t) = c_{\mathcal{X}_0}\sqrt{t}$ where $c_{\mathcal{X}_0} = \frac{1}{\sqrt{d}} \lambda_{\text{min}}(\sum_{x \in \mathcal{X}_0} xx^{\top})$. 
    Consider the rule, defined recursively as: $i_0 = 1$, and for $t \geq 0$, $i_{t+1} = (i_t \mod d) + \mathbbm{1}_{\left\{\lambda_{\min }\left(\sum_{s=1}^t x_s x_s^{\top}\right)<f(t)\right\}}$ and

\begin{align}\label{eq:forced_expl}
    x_{t+1}= \begin{cases}x_0\left(i_t\right) & \text { if } \lambda_{\min }\left(\sum_{s=1}^t x_s x_s^{\top}\right)<f(t) \\ b_t & \text { otherwise. }\end{cases}    
\end{align}

Then for all $t \geq \frac{5 d}{4}+\frac{1}{4 d}+\frac{3}{2}$, we have 

$$\lambda_{\min }\left(\sum_{s=1}^t x_s x_s^{\top}\right) \geq f(t-d-1)$$ 

\end{lemma}

To ensure the MLE converges almost surely to the true parameter $\theta^*$, the ratio between the minimum eigenvalue and the logarithm of the maximum eigenvalue of the matrix $\sum_{s = 1}^t x_sx_s^{\top}$ must tend to infinity \citep{chen1999strong}. This convergence can be guaranteed by the forced exploration component.

\begin{lemma}\label{lemma:strong_cons}
    Under the forced exploration sampling defined in Eq. \eqref{eq:forced_expl}, the MLE estimator converge a.s. to the true parameter

    $$ \lim _{t \rightarrow \infty} \hat{\theta}_t \stackrel{a.s.}{=} \theta^*
    $$

\end{lemma}

Another consequence of forced exploration is the following tail inequality. We can upper bound the probability that the distance between the MLE and the true parameter exceeds $\varepsilon$. This inequality will play an important role to prove the upper bound for the expected value of the sample complexity.

\begin{lemma}\label{lemma:tail_error_th}
    Let $\varepsilon>0$, assume that $\lambda_{\min }\left(\mathbf{A}_t\right) \geq c t^{1/2}$ a.s. for all $t \geq t_0$ for some $t_0 \geq 1$ and for $c > \kappa_0 \lambda_0$. Then
    
    $$
    \forall t \geq t_0 \quad \mathbb{P}\left(\|\theta^{(1)}_t-\theta^*\| \geq \varepsilon\right) \leq c_2 t^{\beta_2} \exp \left(- c_1 t^{\beta_1}\right)
    $$

    Where $c_1, c_2, \beta_1, \beta_2$ are positive constants independent of $\epsilon$ and $t$.
\end{lemma}

\paragraph{Tracking}

As noted in \citet{jedra2020optbailinear, wang2021fast}, we can define the function $\psi(\theta, w)$ to recover the optimal proportions from the lower bound in Theorem \ref{thm:thmsclb}

\begin{align}
    \psi(\theta, w)= \inf _{\lambda \in \Alt(\theta)}
\frac{1}{2}\|\theta-\lambda\|_{\mathbf{H}_w(\theta)}^2
\end{align}

This function has a tractable form for BAI, TBP and top-m best arm identification (See Section \ref{section:examples} in the appendix). The objective is to optimize $\psi(\theta, w)$ over the simplex $\Sigma$ so we can estimate the optimal proportions needed in the TS algorithm.

\begin{align}
    \psi^{\star}(\theta) = \max _{w \in \Sigma} \psi(\theta, w)
\end{align}

\begin{align}\label{eq:optimal_prop}
    C^{\star}(\theta) = \arg \max _{w \in \Sigma} \psi(\theta, w)
\end{align}

Note that $\psi(\theta^*, w^{\star}) = T^{\star}(\theta^*)^{-1}$, where 
$w^{\star} \in C^{\star}(\theta^*)$. As pointed out in previous works, the solution to Eq. \eqref{eq:optimal_prop} may involve multiple optimal proportions. However, similar to \citet{jedra2020optbailinear}, we only need to prove that $\psi^{\star}(\theta)$ is continuous in $\theta$ and that $C^{\star}(\theta)$ is convex to guarantee the algorithm converges to an optimal proportion inside $C^{\star}(\theta^*)$. We will use the Frank-Wolfe algorithm, as previous works, to solve the optimization problem in Eq. \eqref{eq:optimal_prop} \citep{jedra2020optbailinear, degenne2020pureexp}. A caveat of using Frank-Wolfe algorithm is that convergence is not guaranteed when the function $\psi^{\star}(\theta)$ is non-smooth. However, empirical evidence suggests that it can still converge in practice. Moreover, recent work on developing projection-free algorithms for non-smooth functions could be leveraged in the future \citep{asgari2022projection}. 

We use the standard tracking
procedure to track the estimated optimal proportions

\begin{align}\label{eq:track}
    b_t=\underset{x \in \operatorname{supp}\left(\sum_{s=1}^t w(s)\right)}{\arg \min }\left(N_x(t)-\sum_{s=1}^t w_x(s)\right),    
\end{align}

Where $w(t) \in C^{\star}(\theta^{(1)}_t)$, i.e.

\begin{align}\label{eq:optmial_weights}
    w(t) = \arg \max _{w \in \Sigma} \psi (\theta^{(1)}_t, w)    
\end{align}

We can prove that under this sampling rule, the observed proportions of the sampled arms converge to a true optimal proportions.

\begin{proposition}\label{prop:converg_track}
    Under the sampling rules defined by Eq. \eqref{eq:forced_expl} and Eq. \eqref{eq:track}, the proportions of arm draws approach $C^{\star}(\theta^*): \lim _{t \rightarrow \infty} d_{\infty}\left(\left(N_x(t) / t\right)_{x \in \mathcal{X}}, C^{\star}(\theta^*)\right)=0$, a.s..
    
\end{proposition}

\begin{algorithm}[tb]
   \caption{Log Track-and-Stop}
   \label{alg:GLMTaS}
\begin{algorithmic}
   \State {\bfseries Input:} Arms $\mathcal{X}$, confidence level $\delta$;
   \State {\bfseries Initialize:} 
   $t=0$, $i=0$, $\mathbf{A}_0=0$, $Z(0)=0$,$N(0)=(N_x(0))_{x \in \mathcal{X}}=0$;
   \While{$t \notin B$ or $Z(t) < \beta(\delta, t)$}
   \If{$\lambda_{\min}(\mathbf{A}_t) < f(t)$}
    \State select $x$ according to Eq. \eqref{eq:forced_expl}
    \Else
    \State select $x$ according to Eq. \eqref{eq:track}
   \EndIf
   \State $t \leftarrow t+1$, 
    \State sample arm $x$
    \State update $N(t)$, $\hat{\theta}_{t}$, $\theta^{(1)}_t$, $Z(t)$, $\mathbf{A}_{t}$, $\mathbf{H}_{t}$
    \State $w(t) = \arg \max _{w \in \Sigma} \psi (\theta_t^{(1)}, w)$
    \EndWhile
    \State{\bfseries Return $i^{\star}(\theta^{(1)}_{\tau})$}
\end{algorithmic}
\end{algorithm}

Although updating $w(t)$ according to Eq. \eqref{eq:optmial_weights} can be computationally expensive, Log-TS can easily adapted to the lazy approach of \citet{jedra2020optbailinear}, so we do not need to update $w(t)$ at every step.

\section{Sample Complexity of Log-TS}\label{section:sample_complexity} 

In this section we state two upper bounds for the sample complexity under Log-TS when we set $\lambda(t) = c\log (t)$ for some constant $c > 0$. The pseudocode is provided in Algorithm \ref{alg:GLMTaS}. 

\begin{theorem}\label{thm:thm_ubsce}

Logistic Track-and-stop (Log-TS) satisfies the following sample complexity upper bound

$$ \mathbb{P}_{\theta}[\limsup _{\delta \rightarrow 0} \frac{\tau_{\delta}}{(\log(\frac{1}{\delta}))^2} \lesssim T^{\star}(\theta)] = 1$$

\end{theorem}

\begin{theorem}\label{thm:thm_ubsce_exp}

Logistic Track-and-stop (Log-TS) satisfies the following sample complexity upper bound

$$\limsup _{\delta \rightarrow 0} \frac{\mathbb{E}_{\theta}[\tau_{\delta}]}{(\log(\frac{1}{\delta}))^2} \lesssim T^{\star}(\theta)$$

\end{theorem}

As we mention in Section \ref{subsection:consentration} these bounds have an extra factor $\log(\frac{1}{\delta})$ that comes from the threshold $\gamma_t(\delta)$.




\section{Experiments}\label{section:experiments} 

\subsection{BAI}

In this section we evaluate the performance of Log-TS for BAI. To the best of our knowledge, there are only two algorithms that address BAI for the logistic bandit \citep{kazerouni2021,  jun2021improved}. We only compare our algorithm against RAGE-GLM-R and random sampling equipped with the same stopping rule as Log-TS, since RAGE-GLM-R is the state of the art algorithm for this task. Lets define $x^{\star} (\theta):=i^{\star}(\theta) = \arg \max _{x \in \mathcal{X}} \{x^{\top} \theta \}$. Then, to implement the algorithm, we use the following lemma:

\begin{lemma}\label{lemma:example_bai}
    For all $\theta \in \mathbb{R}^d$,

$$
T^{\star}(\theta)^{-1}=\max _{w \in \Sigma} \min _{x \neq x^{\star} (\theta)} \frac{\left(\theta^{\top}x^{\star} (\theta)-\theta^{\top}x\right)^2}{2\left\|x^{\star} (\theta)-x\right\|_{\mathbf{H}_w^{-1}}^2}
$$

\end{lemma}

We evaluated the algorithms under two settings:

\paragraph{Benchmark for BAI in linear bandits.}
The benchmark examples in the linear bandit BAI literature introduced by \citet{soare2014bailinear} is the following. Consider $\mathcal{X} = \{ e_1, \dots, e_d, x'\} \subseteq \mathbb{R}^d$ where $e_i$ is the $i$-standard basis vector, $x' = \cos (\alpha)e_1 + \sin (\alpha)e_2$ with $\alpha$ small,  and $\theta$  proportional to $e_1$ so that $e_1 = \arg \max _{x \in \mathcal{X}} x^{\top}\theta$. This setting is designed to be a hard instance for an algorithm because the rewards $\mu(\theta^{\top}e_1)$ and $\mu(\theta^{\top}x')$ are close, so distinguishing the best arm becomes a difficult task. In Figure \ref{figure:bai-bench}, we observe comparable results between Log-TS and Rage-GLM-R, suggesting that our algorithm is competitive with the state-of-the-art method.

\paragraph{Uniform Distribution on a Sphere.} In this example, $\mathcal{X}$ is sampled from a unit sphere of dimension $d = 2$ centered at the origin with $|\mathcal{X}| = K$. We set $\theta$ randomly so that $\|\theta\| = 1$. To control the variation between experiments, we restricted the complexity $T^{\star}(\theta)^{-1}$ of the experiments across the different number arms, allowing higher complexity for higher value of $K$. In Table \ref{table:bai-unif}, Log-TS shows a distinctive advantage, achieving almost a 5x reduction in sample complexity. This result highlights the efficiency of Log-TS, particularly in scenarios with many arms. The superiority of Log-TS can be attributed to the independence of the stopping rule with respect to the number of arms, allowing it to scale better in large action spaces.

In the implementation of the stopping rule, we adopt a heuristic by using $\lambda_{\min}(\mathbf{H}_t(\hat{\theta})) > \lambda(t)$ instead of $t\in B$ and select $\lambda(t) = d \log(t)$. This choice improves the algorithm’s sample complexity performance avoiding unnecessary exploration. In all instances, all algorithms found the best arm correctly using $\delta = 0.10$.



\begin{table}[h]
\caption{Average sample complexities, expressed in thousands, of uniform distribution for BAI. The results are shown for various numbers of arms $K$. Standard deviations and means were computed across 10 randomized trials.} \label{table:bai-unif}
\begin{center}
\begin{tabular}{cccc}
\toprule
$K$&Log-TS&Rage-GLM-R&Random \\
\midrule
100&\textbf{4.05 (0.35)} & 33.22 (19.14) & 11.44 (1.59) \\
200&\textbf{6.94 (0.39)} & 49.34 (13.26) & 25.10 (5.86)\\
300&\textbf{10.66 (1.44)} & 51.94 (15.96) & 44.32 (6.44)\\
400&\textbf{14.36 (2.31)} & 53.39 (18.40) & 55.64 (11.86)\\
500&\textbf{16.59 (2.32)} & 68.00 (19.13) & 74.10 (14.41)\\
\bottomrule
\end{tabular}
\end{center}
\end{table}


\begin{figure}[h]
\vspace{.1in}
\begin{center}
\includegraphics[width=10cm]{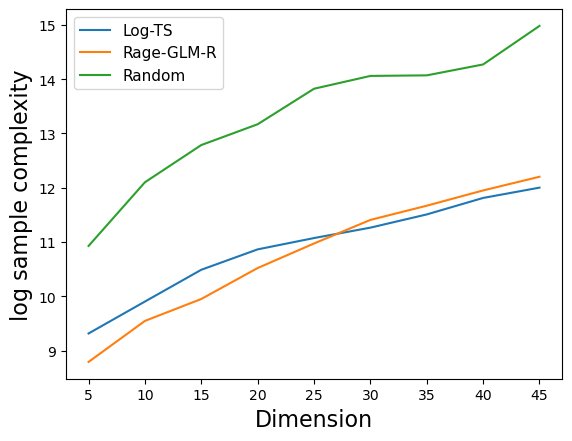}
\end{center}
\vspace{.1in}
\caption{Logarithm of sample complexity of the benchmark setup for BAI against dimension of the action space $\mathcal{X}$.}
\label{figure:bai-bench}
\end{figure}

\subsection{Thresholding Bandit Problem}

In comparison to BAI, there are not many algorithms that can solve TBP. For the linear case \citet{degenne2020pureexp} developed a general algorithm for pure exploration problems called LinGame. Another example is the work in \citet{mason2022nearly}. They developed a nearly optimal algorithm for level set estimation, aka TBP. However, to the best of our knowledge, there is no algorithm for TBP under GLM bandit. Therefore, we compare Log-TS with random sampling, both equipped with the same stopping rule. For TBP, given $\rho \in (0, 1)$, we have $i^{\star}(\theta) = \{x \in \mathcal{X}: \mu(x^{\top} \theta) > \rho\}$. Similarly to BAI, we have an explicit expression that allows us to implement the algorithm:

\begin{lemma}\label{lemma:example_tbp}
    For all $\theta \in \mathbb{R}^d$,

$$
T^{\star}(\theta)^{-1}=\max _{w \in \Sigma} \min _{x \in \mathcal{X}} \frac{(\theta^{\top} x - \mu^{-1}(\rho))^2}{2\left\|x\right\|_{\mathbf{H}_w^{-1}}^2}
$$

\end{lemma}

\paragraph{Benchmark for TBP in linear bandits.} We proposed a slightly modified version of the BAI benchmark setting such that it is hard to distinguish if one of the arms is above or below the threshold. Consider a similar setting as before, $\mathcal{X} = \{ e_1, \dots, e_d, x', x''\} \subseteq \mathbb{R}^d$ where $e_i$ is the $i$-standard basis vector, $x' = p(\cos (\alpha)e_1 + \sin (\alpha)e_2)$, $x'' = (1-p)(\cos (-\alpha)e_1 + \sin (-\alpha)e_2)$ with $\alpha$ small and $p$ close to $1/2$. Let $\rho = (\mu(\theta^{\top} x') +  \mu(\theta^{\top} x''))/2$, $\theta$ proportional to $e_1$, such that 
$x'^{\top}\theta, x''^{\top}\theta \approx \rho$. In Figure \ref{figure:tbp-bench}, we observe a reduction between 5 and 10 times in the sample complexity of Log-TS compare to random sampling.

\textbf{Uniform Distribution on a Sphere.} In the same way as in BAI, $\mathcal{X}$ is sampled from a unit
sphere of dimension $d = 2$ centered at the origin. We set $\theta$ randomly such that $\|\theta\| = 1$, and $\rho = 0.5$. We also control the complexity $T^{\star}(\theta)^{-1}$ of the experiments across the different number arms to avoid high variance. In Table \ref{table:tbp-unif}, it is shown that Log-TS outperforms random sampling by roughly 5 times in sample complexity. This confirms the efficiency of Log-TS in many arms scenarios for different pure exploration problems. 

Algorithms were able to identify the arms above and below the threshold $\rho$ correctly in all instances with $\delta = 0.10$. 


\begin{table}[h]
\caption{Average sample complexities, expressed in thousands, of uniform distribution for TBP. The results are shown for various numbers of arms $K$. Standard deviations and means were computed across 10 randomized trials.} \label{table:tbp-unif}
\begin{center}
\begin{tabular}{ccc}
\toprule
$K$ & Log-TS & Random \\
\midrule
20& \textbf{7.49 (2.36)}& 15.22 (4.06) \\
30& \textbf{9.50 (0.65)}& 37.64 (10.84)\\
40& \textbf{12.86 (0.44)}& 50.84 (12.21)\\
50& \textbf{16.15 (3.64)}& 94.47 (23.06)\\
\bottomrule
\end{tabular}
\end{center}
\end{table}

\begin{figure}[h]
\vspace{.1in}
\begin{center}
\includegraphics[width=10cm]{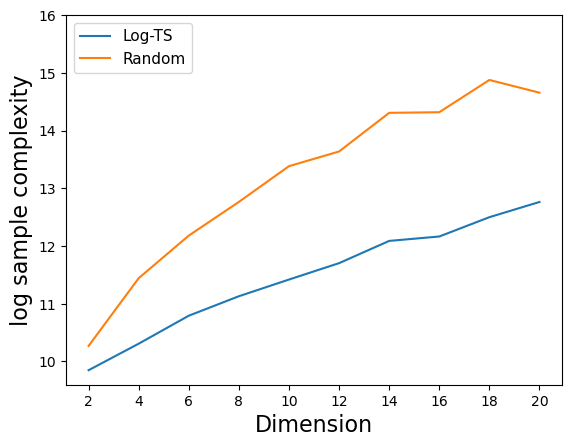}
\end{center}
\vspace{.1in}
\caption{Logarithm of sample complexity of the benchmark setup for TBP against dimension of the action space $\mathcal{X}$.}
\label{figure:tbp-bench}
\end{figure}

\section{Discussion}\label{section:discussion}

To the best of our knowledge, we propose the first track-and-stop algorithm for general pure exploration problems under logistic bandits, proving that our algorithm is near asymptotically optimal. This means that the sample complexity is upper-bounded in the limit up to logarithmic factor, and the upper bound matches the approximated lower bound. We evaluate our algorithm on BAI and TBP, demonstrating its advantages over existing algorithms and random sampling. The incorporation of forced exploration enables us to apply the tail inequality without a warm-up phase, enhancing the algorithm's applicability in real-world scenarios. On the other hand, the computation of the MLE and its projection can make the algorithm slow under certain conditions since it can not be computed recursively, so TS-Log could benefit from online approaches like the one presented in \citet{faury2022jointly}. Additionally, our asymptotic results suffer from an extra logarithmic factor $\log (1/\delta)$. This term arises directly from the tail inequality in Lemma \ref{lemma:concentration}; thus, to eliminate it, a tighter tail inequality is necessary, one that depends on $\gamma_t(\delta) = \mathcal{O}(\sqrt{\log(1/\delta)})$ instead of $\gamma_t(\delta) = \mathcal{O}(\log(1/\delta))$.


\newpage
\bibliographystyle{achemso}
\bibliography{references}


\begin{alphasection}

\section{Concentration}

The proof of Lemma \ref{lemma:concentration} follows tightly \citet{faury2020improved}, with the only difference that there is no regularization. Instead, we assume we can control the minimum eigenvalue of the fisher information matrix. First we will prove the general concentration bound

\begin{lemma} \label{lemma:concentration_2}
    Let $\left\{\mathcal{F}_t\right\}_{t=1}^{\infty}$ be a filtration. Let $\left\{x_t\right\}_{t=1}^{\infty}$ be a stochastic process in $\mathcal{B}_2(d)$ such that $x_t$ is $\mathcal{F}_t$ measurable. Let $\left\{\varepsilon_t\right\}_{t=1}^{\infty}$ be a martingale difference sequence such that $\varepsilon_{t}$ is $\mathcal{F}_{t}$ measurable. Furthermore, assume that conditionally on $\mathcal{F}_t$ we have $\left|\varepsilon_{t}\right| \leq 1$ almost surely, and note $\sigma_t^2:=\mathbb{E}\left[\varepsilon_{t}^2 \mid \mathcal{F}_{t-    1}\right]$. Let $\lambda(t)>0$ and for any $t \geq 1$ define:
$$
\mathbf{H}_t:=\sum_{s=1}^{t} \sigma_s^2 x_s x_s^T, \quad S_t:=\sum_{s=1}^{t} \varepsilon_{s} x_s .
$$

Then. if exist $t_0 \ge 1$ such that for $t \ge t_0$, $\lambda_{\min }(\mathbf{H}_t) > \lambda(t)$, for any $\delta \in(0,1]$ :

$$
\mathbb{P}\left(\exists t \geq 1,\left\|S_t\right\|_{\mathbf{H}_t^{-1}} \geq 2\sqrt{\lambda(t)}+\frac{2}{\sqrt{\lambda(t)}} \log \left(\frac{\operatorname{det}\left(\mathbf{H}_{t}\right)^{\frac{1}{2}} \lambda(t)^{-\frac{d}{2}}}{\delta}\right)+\frac{2}{\sqrt{\lambda(t)}} d \log (2)\right) \leq \delta
$$

\end{lemma}

\begin{proof}
    Let $\bar{\mathbf{H}}_{t} := \mathbf{H}_t - \lambda(t) \mathbf{I}_d$ , for $\xi \in \mathbb{R}^d$ let $M_0(\xi) = 1$ and for $t \ge t_0$ define:

    $$M_t(\xi):=\exp \left(\xi^T S_t-\|\xi\|_{\bar{\mathbf{H}}_t}^2\right)$$

    Using lemma 5 from \citet{faury2020improved}, for all $\xi \in \mathcal{B}_2(d),\left\{M_t(\xi)\right\}_{t=t_0}^{\infty}$ is a non-negative super-martingale.

    Let $h(\xi)$ be a probability density function with support on $\mathcal{B}_2(d)$ (to be defined later). For $t \geq t_0$ let:
    
    $$\bar{M}_t:=\int_{\xi} M_t(\xi) d h(\xi)
    $$

    By Lemma 20.3 of \citet{lattimore2020bandit} $\bar{M}_t$ is also a non-negative super-martingale, and $\mathbb{E}\left[\bar{M}_0\right]=1$. Let $\tau$ be a stopping time with respect to the filtration $\left\{F_t\right\}_{t=0}^{\infty}$. We can follow the proof of Lemma 8 in \citet{abbasi2011improved} to justify that $\bar{M}_\tau$ is well-defined (independently of whether $\tau<\infty$ holds or not) and that $\mathbb{E}\left[\bar{M}_\tau\right] \leq 1$. Therefore, with $\delta \in(0,1)$ and thanks to the maximal inequality:

    \begin{align}\label{eq:max_martingale}
        \mathbb{P}\left(\log \left(\bar{M}_\tau\right) \geq \log \left(\frac{1}{\delta}\right)\right)=\mathbb{P}\left(\bar{M}_\tau \geq \frac{1}{\delta}\right) \leq \delta  
    \end{align}

    We now proceed to compute $\bar{M}_t$ (more precisely a lower bound on  $\bar{M}_t$). Let $\lambda(t)$ be a strictly positive scalar, and set $h$ to be the density of an isotropic normal distribution of precision $2\lambda(t)$ truncated on $\mathcal{B}_2(d)$. We will denote $N(h)$ its normalization constant. Simple computations show that:

    $$
    \begin{aligned}
        \bar{M}_t &= \frac{1}{N(h)} \int_{\mathcal{B}_2(d)} \exp \left(\xi^T S_t-\|\xi\|_{\mathbf{H}_t}^2-\lambda(t)\|\xi\|^2\right) d \xi \\
         &\ge \frac{1}{N(h)} \int_{\mathcal{B}_2(d)} \exp \left(\xi^T S_t-2\|\xi\|_{\mathbf{H}_t}^2\right) d \xi
    \end{aligned}
    $$

To ease notations, let $f(\xi):=\xi^T S_t-2\|\xi\|_{\mathbf{H}_t}^2$ and $\xi_*=\arg \max _{\|\xi\|_2 \leq 1 / 2} f(\xi)$. Because:
$$
f(\xi)=f\left(\xi_*\right)+\left(\xi-\xi_*\right)^T \nabla f\left(\xi_*\right)-2\left(\xi-\xi_*\right)^T \mathbf{H}_t\left(\xi-\xi_*\right)
$$
we obtain that:
$$
\begin{aligned}
\bar{M}_t & =\frac{e^{f\left(\xi_*\right)}}{N(h)} \int_{\mathbb{R}^d} \mathbf{1}_{\|\xi\|_2 \leq 1} \exp \left(\left(\xi-\xi_*\right)^T \nabla f\left(\xi_*\right)-2\left(\xi-\xi_*\right)^T \mathbf{H}_t\left(\xi-\xi_*\right)\right) d \xi \\
& =\frac{e^{f\left(\xi_*\right)}}{N(h)} \int_{\mathbb{R}^d} \mathbf{1}_{\left\|\xi+\xi_*\right\|_2 \leq 1} \exp \left(\xi^T \nabla f\left(\xi_*\right)-2\xi^T \mathbf{H}_t \xi\right) d \xi \hspace{28pt} \text { (change of variable } \xi+\xi_* \text { ) } \\
& \geq \frac{e^{f\left(\xi_*\right)}}{N(h)} \int_{\mathbb{R}^d} \mathbf{1}_{\|\xi\|_2 \leq 1 / 2} \exp \left(\xi^T \nabla f\left(\xi_*\right)-2\xi^T \mathbf{H}_t \xi\right) d \xi \hspace{35pt} \text { (as }\left\|\xi_*\right\|_2 \leq 1 / 2 \text { ) } \\
& =\frac{e^{f\left(\xi_*\right)}}{N(h)} \int_{\mathbb{R}^d} \mathbf{1}_{\|\xi\|_2 \leq 1 / 2} \exp \left(\xi^T \nabla f\left(\xi_*\right)\right) \exp \left(-\frac{1}{2} \xi^T\left(4 \mathbf{H}_t\right) \xi\right) d \xi
\end{aligned}
$$

By defining $g(\xi)$ the density of the normal distribution of precision $4 \mathbf{H}_t$ truncated on the ball $\left\{\xi \in \mathbb{R}^d,\|\xi\|_2 \leq 1 / 2\right\}$ and noting $N(g)$ its normalizing constant, we can rewrite:
$$
\begin{aligned}
\bar{M}_t & \geq \exp \left(f\left(\xi_*\right)\right) \frac{N(g)}{N(h)} \mathbb{E}_g\left[\exp \left(\xi^T \nabla f\left(\xi_*\right)\right)\right] \\
& \geq \exp \left(f\left(\xi_*\right)\right) \frac{N(g)}{N(h)} \exp \left(\mathbb{E}_g\left[\xi^T \nabla f\left(\xi_*\right)\right]\right) \hspace{45pt} \text { (Jensen's inequality) } \\
& \geq \exp \left(f\left(\xi_*\right)\right) \frac{N(g)}{N(h)} \hspace{140pt} \text { (as } \mathbb{E}_g[\xi]=0 \text { ) }
\end{aligned}
$$

Unpacking this results and assembling (10) and (11), we obtain that for any $\xi_0$ such that $\left\|\xi_0\right\|_2 \leq 1 / 2$ :

\begin{equation}\label{eq:proba_martingale}
\begin{aligned}[b]
\mathbb{P}\left(\bar{M}_t \geq \frac{1}{\delta}\right) & \geq \mathbb{P}\left(\exp \left(f\left(\xi_*\right)\right) \frac{N(g)}{N(h)} \geq 1 / \delta\right) \\
& =\mathbb{P}\left(\log \left(\exp \left(f\left(\xi_*\right)\right) \frac{N(g)}{N(h)}\right) \geq \log (1 / \delta)\right) \\
& =\mathbb{P}\left(f\left(\xi_*\right) \geq \log (1 / \delta)+\log \left(\frac{N(h)}{N(g)}\right)\right) \\
& =\mathbb{P}\left(\max _{\|\xi\|_2 \leq 1 / 2} \xi^T S_t-2\|\xi\|_{\mathbf{H}_t}^2 \geq \log (1 / \delta)+\log \left(\frac{N(h)}{N(g)}\right)\right) \\
& \geq \mathbb{P}\left(\xi_0^T S_t-2\left\|\xi_0\right\|_{\mathbf{H}_t}^2 \geq \log (1 / \delta)+\log \left(\frac{N(h)}{N(g)}\right)\right)
\end{aligned}
\end{equation}

In particular, we can use:
$$
\xi_0:=\frac{\mathbf{H}_t^{-1} S_t}{\left\|S_t\right\|_{\mathbf{H}_t^{-1}}} \frac{\lambda(t)^{1/2}}{4}
$$
since
$$
\left\|\xi_0\right\|_2 \leq \frac{\lambda(t)^{1/2}}{4}\left(\lambda_{\min }\left(\mathbf{H}_t\right)\right)^{-1 / 2} \leq 1 / 2
$$

Using this value of $\xi_0$ in Eq. \eqref{eq:proba_martingale} yields:
$$
\mathbb{P}\left(\left\|S_t\right\|_{\mathbf{H}_t^{-1}} \geq \frac{\sqrt{\lambda(t)}}{2}+\frac{4}{\sqrt{\lambda(t)}} \log (1/\delta)+\frac{4}{\sqrt{\lambda(t)}} \log \left(\frac{N(h)}{ N(g)}\right)\right) \leq \mathbb{P}\left(\bar{M}_t \geq \frac{1}{\delta}\right)
$$

We can use the upper-bound from lemma 6 of \citet{faury2020improved} for the log of their ratio $\log \left(\frac{N(h)}{N(g)}\right)$. Therefore with probability at least $1-\delta$ and by using Eq. \eqref{eq:max_martingale}:
$$
\left\|S_\tau\right\|_{\mathbf{H}_\tau^{-1}} \leq \frac{\sqrt{\lambda(\tau)}}{2}+\frac{4}{\sqrt{\lambda(\tau)}} \log (1 / \delta)+\frac{4}{\sqrt{\lambda(\tau)}} \log \left(\frac{2^{d} \operatorname{det}\left(\mathbf{H}_{\tau}\right)^{1 / 2}}{\lambda(\tau)^{d/2} \delta}\right)+\frac{4}{\sqrt{\lambda(\tau)}} d \log (2)
$$

Directly following the stopping time construction argument in the proof of Theorem 1 of (Abbasi-Yadkori et al., 2011) we obtain that with probability at least $1-\delta$, for all $t \in \mathbb{N}$ :
$$
\left\|S_t\right\|_{\mathbf{H}_t^{-1}} \leq \frac{\sqrt{\lambda(t)}}{2}+\frac{4}{\sqrt{\lambda(t)}}\log \left(\frac{2^{d} \operatorname{det}\left(\mathbf{H}_t\right)^{1 / 2}}{\lambda(t)^{d/2} \delta}\right)+\frac{4}{\sqrt{\lambda(t)}} d \log (2)
$$

\end{proof}

\begin{lemma} 
(lemma 6 \citep{faury2020improved})
    The following inequality holds:
    $$
    \log \left(\frac{N(h)}{N(g)}\right) \leq \log \left(2^{d} \left(\frac{\operatorname{det}\left(\mathbf{H}_t\right)}{\lambda(t)^d}\right)^{1 / 2}\right)+d \log (2)
    $$
\end{lemma}

\textbf{Lemma \ref{lemma:concentration}}
    Let $\delta \in(0,1]$ and $\lambda(t) > 0$ for $ t \ge 1$. If exist $t_0 \ge 1$ such that for $t \ge t_0$, $\lambda_{\min }(H_t\left(\theta_*\right)) > \lambda(t)$, with probability at least $1-\delta$:
    
    $$
    \forall t \geq t_0, \quad\left\|g_t(\hat{\theta}_t)-g_t\left(\theta_*\right)\right\|_{\mathbf{H}_t^{-1}\left(\theta_*\right)} \leq \gamma_t(\delta)
    $$

    Where $\gamma_t(\delta):=\frac{\sqrt{\lambda(t)}}{2}+\frac{4}{\sqrt{\lambda(t)}} \log \left(\frac{2^d}{\delta}\left(\frac{L t}{d \lambda(t)}\right)^{\frac{d}{2}}\right)$

\begin{proof}
    Recall that $\hat{\theta}_t$ is the unique maximizer of the log-likelihood:
$$
\mathcal{L}_t^\lambda(\theta):=\sum_{s=1}^{t}\left[r_{s} \log \mu\left(x_s^{\top} \theta\right)+\left(1-r_{s}\right) \log \left(1-\mu\left(x_s^{\top} \theta\right)\right)\right]
$$
and therefore $\hat{\theta}_t$ is a critical point of $\mathcal{L}_t^\lambda(\theta)$. Solving for $\nabla_\theta \mathcal{L}_t^\lambda=0$ and using the fact that $\dot{\mu}=\mu(1-\mu)$ we obtain:
$$
\sum_{s=1}^{t} \mu\left(x_s^{\top}\hat{\theta}_t \right) x_s=\sum_{s=1}^{t} r_{s} x_s
$$

This result, combined with the definition of $g_t\left(\theta_*\right)=\sum_{s=1}^{t-1} \mu\left(x_s^{\top} \theta_*\right) x_s$ yields:
$$
\begin{aligned}
g_t(\hat{\theta}_t)-g_t\left(\theta_*\right) & =\sum_{s=1}^{t} \varepsilon_{s} x_s \\
& =S_t
\end{aligned}
$$
where we denoted $\varepsilon_{s}:=r_{s}-\mu\left(x_s^{\top} \theta_*\right)$ for all $s \geq 1$ and $S_t:=\sum_{s=1}^{t} \varepsilon_{s} x_s$ for all $t \geq 1$. Then:

\begin{align}\label{eq:assem}
    \left\|g_t(\hat{\theta}_t)-g_t\left(\theta_*\right)\right\|_{\mathbf{H}_t^{-1}\left(\theta_*\right)} = \left\|S_t\right\|_{\mathbf{H}_t^{-1}\left(\theta_*\right)}    
\end{align}

Note that $\left\{\varepsilon_t\right\}_{t=1}^{\infty}$ is a martingale difference sequence adapted to $\mathcal{F}$ and almost surely bounded by 1 . Also, note that for all $s \geq 1$ :
$$
\dot{\mu}\left(x_s^{\top}\theta_*\right)=\mu\left(x_s^{\top} \theta_*\right)\left(1-\mu\left(x_s^{\top} \theta_*\right)\right)=\mathbb{E}\left[\varepsilon_{s}^2 \mid \mathcal{F}_{s-1}\right]=: \sigma_s^2
$$
and thus $\mathbf{H}_t\left(\theta_*\right)=\sum_{s=1}^{t} \sigma_s^2 x_s x_s^{\top}$. All the conditions of Lemma \ref{lemma:concentration_2} are checked and therefore:

\begin{equation}\label{eq:prob_lb}
\begin{aligned}[b]
1-\delta & \leq \mathbb{P}\left(\forall t \geq 1,\left\|S_t\right\|_{\mathbf{H}_t^{-1}\left(\theta_*\right)} \leq \frac{\sqrt{\lambda(t)}}{2}+\frac{4}{\sqrt{\lambda(t)}} \log \left(\frac{\operatorname{det}\left(\mathbf{H}_t\left(\theta_*\right)\right)^{1 / 2}}{\lambda(t)^{d/2} \delta}\right)+\frac{4 d}{\sqrt{\lambda(t)}} \log (2)\right) \\
& \leq \mathbb{P}\left(\forall t \geq 1,\left\|S_t\right\|_{\mathbf{H}_t^{-1}\left(\theta_*\right)} \leq \frac{\sqrt{\lambda(t)}}{2}+\frac{4}{\sqrt{\lambda(t)}} \log \left(\frac{(L t / d)^{d / 2}}{\lambda(t)^{d/2}\delta}\right)+\frac{4 d}{\sqrt{\lambda(t)}} \log (2)\right) \\
& \leq \mathbb{P}\left(\forall t \geq 1,\left\|S_t\right\|_{\mathbf{H}_t^{-1}\left(\theta_*\right)} \leq \frac{\sqrt{\lambda(t)}}{2}+\frac{4}{\sqrt{\lambda(t)}} \log \left(\frac{1}{\delta}\left(\frac{L t}{d \lambda(t)}\right)^{d / 2}\right)+\frac{4 d}{\sqrt{\lambda(t)}} \log (2)\right) \\
& =\mathbb{P}\left(\forall t \geq 1,\left\|S_t\right\|_{\mathbf{H}_t^{-1}\left(\theta_*\right)} \leq \gamma_t(\delta)\right)
\end{aligned}    
\end{equation}

where we used that:
$$
\operatorname{det}\left(\mathbf{H}_t\left(\theta_*\right)\right) \leq L^d \operatorname{det}\left(\sum_{s=1}^{t} x_s x_s^{\top}\right) \leq L^d\left(\frac{t}{d}\right)^d \leq\left(\frac{L t}{d}\right)^d
$$
thanks to Lemma \ref{lemma:det-trace-inneq}. Assembling Eq. \eqref{eq:assem} with Eq. \eqref{eq:prob_lb} yields:
$$
\begin{aligned}
\mathbb{P}\left(\forall t \geq 1,\left\|g_t(\hat{\theta}_t)-g_t\left(\theta_*\right)\right\|_{\mathbf{H}_t^{-1}\left(\theta_*\right)} \leq \gamma_t(\delta)\right) & = \mathbb{P}\left(\forall t \geq 1,\left\|S_t\right\|_{\mathbf{H}_t^{-1}\left(\theta_*\right)}\leq \gamma_t(\delta)\right) \\
& \geq 1-\delta
\end{aligned}
$$
hence the announced result.
\end{proof}

For next results, we will use the following notations:
$$
\begin{aligned}
\alpha\left(x, \theta_1, \theta_2\right) & :=\int_{v=0}^1 \dot{\mu}\left(v x^{\top} \theta_2+(1-v) x^{\top} \theta_1\right) d v>0 \\
\mathbf{G}_t\left(\theta_1, \theta_2\right) & :=\sum_{s=1}^{t-1} \alpha\left(x, \theta_1, \theta_2\right) x_s x_s^{\top} \mathbf{I}_d
\end{aligned}
$$
where $\theta_1, \theta_2$ and $x$ are vectors in $\mathbb{R}^d$.
The quantities $\alpha\left(x, \theta_1, \theta_2\right)$ and $\mathbf{G}_t\left(\theta_1, \theta_2\right)$ naturally arise when studying GLMs. Indeed, note that for all $x \in \mathbb{R}^d$ and $\theta \in \mathbb{R}^d$, the following equality holds:
$$
\mu\left(x^{\boldsymbol{\top}} \theta_1\right)-\mu\left(x^{\top} \theta_2\right)=\alpha\left(x, \theta_2, \theta_1\right) x^{\boldsymbol{\top}}\left(\theta_1-\theta_2\right)
$$

This result is classical (see \citet{filippi2010parametric}) and can be obtained by a straight-forward application of the mean-value theorem. It notably allows us to link $\theta_1-\theta_2$ with $g_t\left(\theta_1\right)-g_t\left(\theta_2\right)$. Namely, it is straightforward that:
$$
\begin{aligned}
g_t\left(\theta_1\right)-g_t\left(\theta_2\right) & =\sum_{s=1}^{t-1} \alpha\left(x_s, \theta_2, \theta_1\right) x_s x_s^{\boldsymbol{\top}}\left(\theta_1-\theta_2\right) \\
& =\mathbf{G}_t\left(\theta_2, \theta_1\right)\left(\theta_1-\theta_2\right)
\end{aligned}
$$

Because $\mathbf{G}_t\left(\theta_1, \theta_2\right) \succ \mathbf{0}_{d \times d}$ this yields:
$$
\left\|\theta_1-\theta_2\right\|_{\mathbf{G}_t\left(\theta_2, \theta_1\right)}=\left\|g_t\left(\theta_1\right)-g_t\left(\theta_2\right)\right\|_{\mathbf{G}_t^{-1}\left(\theta_2, \theta_1\right)}
$$

\begin{lemma}
    (Lemma 10 \citep{faury2020improved})For all $\theta_1, \theta_2 \in \Theta$ the following inequalities hold:
$$
\begin{aligned}
& \mathbf{G}_t\left(\theta_1, \theta_2\right) \geq(1+2 S)^{-1} \mathbf{H}_t\left(\theta_1\right) \\
& \mathbf{G}_t\left(\theta_1, \theta_2\right) \geq(1+2 S)^{-1} \mathbf{H}_t\left(\theta_2\right)
\end{aligned}
$$
\end{lemma}

\begin{lemma}\label{lemma:selfconcordance}
    $$
    \begin{aligned}
    \|\theta^{(1)}_t-\theta^*\|_{\mathbf{H}_t(\theta^{(1)}_t)} \le 2\left(1+2S\right) \|g_t(\hat{\theta}_t)-g_t(\theta^*)\|_{\mathbf{H}_t^{-1}(\theta^*)}
    \end{aligned}  
    $$
\end{lemma}

\begin{proof}
    $$
    \begin{aligned}
    \|\theta_t^{(1)}-\theta^*\|_{\mathbf{H}_t(\theta_t^{(1)})} & \le \sqrt{1+2S}\|\theta_t^{(1)}-\theta^*\|_{\mathbf{G}_t(\theta^*, \theta_t^{(1)})} \\
    & = \sqrt{1+2S}\|g_t(\theta_t^{(1)})-g_t(\theta^*)\|_{\mathbf{G}_t^{-1}(\theta^*,\theta_t^{(1)})} \\
    & \le \sqrt{1+2S} \left( \|g_t(\theta_t^{(1)})-g_t(\hat{\theta}_t)\|_{\mathbf{G}_t^{-1}(\theta^*,\theta_t^{(1)})} + \|g_t(\hat{\theta}_t)-g_t(\theta^*)\|_{\mathbf{G}_t^{-1}(\theta^*,\theta_t^{(1)})} \right)\\
    & \le (1+2S) \left( \|g_t(\theta_t^{(1)})-g_t(\hat{\theta}_t)\|_{\mathbf{H}_t^{-1}(\theta_t^{(1)})} + \|g_t(\hat{\theta}_t)-g_t(\theta^*)\|_{\mathbf{H}_t^{-1}(\theta^*)} \right)\\
    & \le 2(1+2S)\|g_t(\hat{\theta}_t)-g_t(\theta^*)\|_{\mathbf{H}_t^{-1}(\theta^*)} \\
    \end{aligned}  
    $$
\end{proof}

\textbf{Lemma \ref{lemma:tail_error_th}}

Let $\epsilon>0$, assume that $\lambda_{\min }\left(\sum_{s=1}^t x_s x_s^{\top}\right) \geq c t^\alpha \ge \kappa_0 \lambda_0$ a.s. for all $t \geq t_0$, some $t_0 \geq 1$ and some constants $\alpha,c>0$. Then
    $$
    \forall t \geq t_0 \quad \mathbb{P}\left(\|\hat{\theta}_t-\theta^*\| \geq \varepsilon\right) \leq c_2 t^{d/2} \exp \left(- c_1 \varepsilon t^{\alpha/2}\right) 
    $$

    Where $c_1, c_2$ are positive constants independent of $\varepsilon$ and $t$.

\begin{proof}
    Similar to Lemma \ref{lemma:selfconcordance}, we have

    $$
    \begin{aligned}
    \|\hat{\theta}_t-\theta^*\|
    &\le 
    \frac{\|\hat{\theta}_t-\theta^*\|_{\mathbf{H}_t(\theta^*)}}{\sqrt{\lambda_{\min}(\mathbf{H}_t(\theta^*))}} \\
    &\le \frac{2(1+2S)}{\sqrt{\lambda_{\min}(\mathbf{H}_t(\theta^*))}}\|g_t(\hat{\theta}_t)-g_t(\theta^*)\|_{\mathbf{H}_t^{-1}(\theta^*)} \\
    &\le \frac{2(1+2S)}{\sqrt{c/\kappa_0}} t^{-\alpha/2} \|g_t(\hat{\theta}_t)-g_t(\theta^*)\|_{\mathbf{H}_t^{-1}(\theta^*)}
    \end{aligned}  
    $$

    And, with probability at least $1-\delta$

    $$
    \begin{aligned}
    \|\hat{\theta}_t-\theta^*\| &
    &\le \frac{2(1+2S)}{\sqrt{c/\kappa_0}} t^{-\alpha/2} \gamma_t(\delta)
    \end{aligned}  
    $$

   which we may rewrite after substitution as

    $$
    \mathbb{P}\left( \|\hat{\theta}_t-\theta^*\| > \varepsilon \right) < c_2 t^{d/2} \exp \left(- c_1 \varepsilon t^{\alpha/2}\right) 
    $$
    
\end{proof}

\textbf{Lemma \ref{lemma:delta_correct}}. 

Under any sampling rule, we have
    $$\mathbb{P}_\theta\left(\tau_\delta<\infty \wedge i^{\star}(\theta^{(1)}_{\tau_\delta}) \neq i^{\star}(\theta)\right) \leq \delta$$

\begin{proof}

Lets consider 

$$\mathcal{E}_1 = \{ \tau_\delta < \infty \} = \{ \exists t \geq 1 : Z(t) > \beta(\delta, t), t \in B\} $$

$$\mathcal{E}_2 = \{ i^{\star}(\theta^{(1)}_{\tau_\delta})\neq i^{\star}(\theta) \} = \{\theta 
\in \Alt(\theta^{(1)}_{\tau_\delta}) \}$$


$$
\begin{aligned}
\mathcal{E}_1 \cap \mathcal{E}_2 & =\left\{  \exists t \geq 1 : Z(t) > \beta(\delta, t), t \in B, \theta 
\in \Alt(\theta^{(1)}_{t}) \right\} \\
& = \left\{  \exists t \geq 1 : \inf _{\lambda \in \Alt(\theta^{(1)}_{t})} \frac{1}{2}\|\theta^{(1)}_{t}-\lambda\|_{\mathbf{H}_t(\theta^{(1)}_{t})}^2 > \beta(\delta, t),  t \in B, \theta 
\in \Alt(\theta^{(1)}_{t})\right\} \\
& \subseteq \left\{  \exists t \geq 1 : \frac{1}{2}\|\theta^{(1)}_{t}-\theta\|_{\mathbf{H}_t(\theta^{(1)}_{t})}^2 > \beta(\delta, t),   t \in B \right\} \\
& \subseteq \left\{  \exists t \geq 1 : \sqrt{2}\left(1+2S\right)\|g_t(\hat{\theta}_t)-g_t(\theta^*)\|_{\mathbf{H}_t^{-1}(\theta^*)} > \sqrt{2}\left(1+2S\right)\gamma_t(\delta), t \in B \right\}
.
\end{aligned}
$$

Where the first containment is because of Lemma \ref{lemma:selfconcordance}. Using Lemma \ref{lemma:concentration} we have

$$
\begin{aligned}
\mathbb{P}_{\theta}[\mathcal{E}_1 \cap \mathcal{E}_2 ]<\delta
\end{aligned}
$$

Thus

$$
\begin{aligned}
\mathbb{P}_\theta\left(\tau_\delta<\infty \wedge i^{\star}(\theta^{(1)}_{\tau_\delta}) \neq i^{\star}(\theta)\right) = \mathbb{P}_{\theta}[\mathcal{E}_1 \cap \mathcal{E}_2 ]<\delta
\end{aligned}
$$

\end{proof}

\textbf{Lemma \ref{lemma:strong_cons}}

    Under forced exploration, the MLE estimator convergence a.s. to the true parameter

    $$
    \lim _{t \rightarrow \infty} \hat{\theta}_t \stackrel{a.s.}{=} \theta^*
    $$

\begin{proof}
    Using Lemma \ref{lemma:forced_expl} we have that exists $t_0$ such that $\lambda_{\min }\left(\sum_{s=1}^t x_s x_s^{\top}\right) \geq c_{\mathcal{X}_0}\sqrt{t}$ if $t \ge t_0$. On the other hand, 

    $$
    \begin{aligned}
        \lambda_{\max }\left(\sum_{s=1}^t x_s x_s^{\top}\right) &\le \operatorname{Tr} (\sum_{s=1}^t x_s x_s^{\top}) \\
        &= \sum_{s=1}^t\operatorname{Tr} (x_s x_s^{\top})  \\
        &= \sum_{s=1}^t \|x_s\|^2  \\
        &\le t
    \end{aligned}
    $$
    
    then

        $$
        \lim _{t \rightarrow \infty} \frac{\lambda_{\min}\left(\sum_{s=1}^t x_s x_s^{\top}\right)}{\log\left(\lambda_{\max}\left(\sum_{s=1}^t x_s x_s^{\top}\right)\right)} = \infty
        $$

    Using Theorem 2 of \citet{chen1999strong} we have $\lim _{t \rightarrow \infty} \hat{\theta}_t \stackrel{a.s.}{=} \theta^*$.
        
\end{proof}

\section{Tracking}

$$
    \psi(\theta, w)= \inf _{\lambda \in \Alt(\theta)}
\frac{1}{2}\|\theta-\lambda\|_{\mathbf{H}_w(\theta)}^2
$$

\begin{lemma} \label{lemma:lemma_contpsi}
If exists $\epsilon_0 > 0$ such that $i^{\star}(\theta_t) =i^{\star}(\theta)$ for all $\theta_t$ such that $\|\theta_t-\theta\|<\epsilon_0$, then $\psi$ is continuous in both $\theta$ and $w$, and $w \mapsto \psi(\theta, w)$ attains its maximum in $\Sigma$ at a point $w_\theta^{\star}$ such that $\sum_{x \in \mathcal{X}}\left(w_\theta^{\star}\right)_x x x^{\top}$ is invertible.

\end{lemma}

\textbf{Remark.} It is easy to check that BAI and TBP meet the assumption that exists $\epsilon_0 > 0$ such that $i^{\star}(\theta_t) =i^{\star}(\theta)$ for all $\theta_t$ that $\|\theta_t-\theta\|<\epsilon_0$ 

\begin{proof}
Let $\theta \in \mathbb{R}^d$ such that $i^{\star}(\theta)$ is unique. Consider the alternative set $\operatorname{Alt}(\theta)$ and denote
$$
f(\theta, \lambda, w)=\frac{1}{2}(\theta-\lambda)^{\top}\left(\sum_{x \in \mathcal{X}} w_x \dot{\mu} (x^{\top} \theta) x x^{\top}\right)(\theta-\lambda) .
$$

Let $\left(\theta_t, w_t\right)_{t \geq 1}$ be a sequence taking values in $\Theta \times \Sigma$ and converging to $(\theta, w)$. Let $\epsilon < \epsilon_0$ and $t_1 \geq 1$ such that for all $t \geq t_1$ we have $\left\|\left(\theta_t, w_t\right)-(\theta, w)\right\|<\epsilon$. Now, $\operatorname{Alt}(\theta_t)=\operatorname{Alt}(\theta)$ because $i^{\star}(\theta)$ is unique and $i^{\star}(\theta_t) =i^{\star}(\theta)$. Furthermore, note that $f(\theta, \lambda, w)$ is a continuous function in $\theta, \lambda, w$, thus it is in particular continuous in $\theta, w$, and there exists $t_2 \geq 1$ such that for all $t \geq t_2$ and for all $\lambda \in \mathbb{R}^d$, it holds that $\left|f\left(\theta, \lambda, w\right) - f\left(\theta_t, \lambda, w_t\right)\right| \leq \epsilon f\left(\theta, \lambda, w\right)$. Hence, with our choice of $\epsilon$, we have for all $t \geq t_1 \vee t_2$

$$
\begin{aligned}
\left|\psi(\theta, w)-\psi\left(\theta_t, w_t\right)\right| & =\left|\min _{\lambda \in \operatorname{Alt}(\theta)} f(\theta, \lambda, w)-\min _{\lambda \in \operatorname{Alt}(\theta_t)} f\left(\theta_t, \lambda, w_t\right)\right| \\
& \leq \epsilon\left|\min _{\lambda \in \operatorname{Alt}(\theta)} f(\theta, \lambda, w)\right| \\
& \leq \epsilon|\psi(\theta, w)| .
\end{aligned}
$$

Thus $\psi$ is continuos in $\theta, w$.
Now, we know that $w \mapsto \psi(\theta, w)$ is continuous on $\Sigma$, and by compactness of the simplex, the maximum is attained at some $w_\theta^{\star} \in \Sigma$. Furthermore, since $\mathcal{X}$ spans $\mathbb{R}^d$, we may construct an allocation $\tilde{w}$ such that $\sum_{a \in \mathcal{A}} \tilde{w}_x x x^{\top}$ is a positive definite matrix. In addition, by construction of $\operatorname{Alt}(\theta)$, there exists some $M>0$ such that for all $\lambda \in \operatorname{Alt}(\theta)$ we have $\|\theta-\lambda\|>M$, which implies that $\psi(\theta, \tilde{w}) \geq M^2 \lambda_{\min }\left(\sum_{x \in \mathcal{X}} \tilde{w}_x x x^{\top}\right)>0$. On the other hand, for any allocation $w \in \Sigma$ such that $\sum_{x \in \mathcal{X}} w_x x x^{\top}$ is rank deficient, we may find a $\lambda \in \operatorname{Alt}(\theta)$ where $\lambda-\theta$ is in the null space of $\sum_{x \in \mathcal{X}} w_x x x^{\top}$. Therefore, $\sum_{x \in \mathcal{X}}\left(w_\theta^{\star}\right)_x x x^{\top}$ is invertible.

\end{proof}

\begin{lemma}\label{lemma:lemma_maxthm}
(Maximum theorem) Let $\theta \in \mathbb{R}^d$. Define

$$\psi^*(\theta)=\max _{w \in \Sigma} \psi(\theta, w)$$

and $C^{\star}(\theta)=\arg \max _{w \in \Sigma} \psi(\theta, w)$. Then $\psi^{\star}$ is continuous at $\theta$, and $C^{\star}(\theta)$ is convex, compact and non-empty. Furthermore, we have for any open neighborhood $\mathcal{V}$ of $C^{\star}(\theta)$, there exists an open neighborhood $\mathcal{U}$ of $\theta$, such that for all $\theta^{\prime} \in \mathcal{U}$, we have $C^{\star}\left(\theta^{\prime}\right) \subseteq \mathcal{V}$.

\end{lemma}

\begin{lemma}\label{lemma:tracking_c} 
(Lemma 6 \citep{jedra2020optbailinear})
    Let $(w(t))_{t \geq 1}$ be a sequence taking values in $\Sigma$, such that there exists a compact, convex and non empty subset $C \subseteq \Sigma$, there exists $\varepsilon>0$ and $t_0(\varepsilon) \geq 1$ such that $\forall t \geq t_0, d_{\infty}(w(t), C) \leq \varepsilon$. Consider a sampling rule defined by Eq. \eqref{eq:forced_expl} and

$$
b_t=\underset{x \in \operatorname{supp}\left(\sum_{s=1}^t w(s)\right)}{\arg \min }\left(N_x(t)-\sum_{s=1}^t w_x(s)\right), 
$$

where $N_x(0)=0$ and for $t \geq 0, N_x(t+1)=N_x(t)+\mathbbm{1}_{\left\{x_t=x\right\}}$.
Then there exists $t_1(\varepsilon) \geq t_0(\varepsilon)$ such that $\forall t \geq t_1(\varepsilon)$, $d_{\infty}\left(\left(N_x(t) / t\right)_{x \in \mathcal{X}}, C\right) \leq\left(p_t+d-1\right) \varepsilon$ where $p_t=$ $\left|\operatorname{supp}\left(\sum_{s=1}^t w(s)\right) \backslash \mathcal{X}_0\right| \leq K-d$.

\end{lemma}

\textbf{Proposition \ref{prop:converg_track}}
    Under the sampling rules Eq. \eqref{eq:forced_expl} and Eq. \eqref{eq:track}, the proportions of arm draws approach $C^{\star}(\theta^*): \lim _{t \rightarrow \infty} d_{\infty}\left(\left(N_x(t) / t\right)_{x \in \mathcal{X}}, C^{\star}(\theta^*)\right)=0$, a.s..

\begin{proof}

Let $\varepsilon>0$. First, by Lemma \ref{lemma:lemma_maxthm}, there exists $\xi(\varepsilon)>0$ such that for all $\theta^{\prime} \in \mathbb{R}^d$ such that $\left\|\theta-\theta^{\prime}\right\|<\xi(\varepsilon)$, it holds that $\max _{w \in C^{\star}\left(\theta^{\prime}\right)} d_{\infty}\left(w, C^{\star}(\theta)\right)<\varepsilon / 2$.
By Lemma \ref{lemma:forced_expl}, we have a sufficient exploration. That is $\liminf _{t \rightarrow \infty} t^{-1 / 2} \lambda_{\min }\left(\sum_{s=1}^t x_s x_s^{\top}\right)>0$. Thus, by Lemma \ref{lemma:strong_cons} , $\hat{\theta}_t$ converges almost surely to $\theta^*$ with a rate of order $o\left(t^{1 / 4}\right)$. Consequently, there exists $t_0 \geq 0$ such that for all $t \geq t_0$, we have $\left\|\theta-\hat{\theta}_t\right\| \leq \xi(\varepsilon)$.Then, we have
$$
d_{\infty}\left(w(t), C^{\star}(\theta^*)\right) \leq \max _{w \in C^{\star}(\hat{\theta}_{t})} d_{\infty}\left(w, C^{\star}(\mu)\right)<\varepsilon .
$$

We have shown that $d_{\infty}\left(w(t), C^{\star}(\theta^*)\right) \underset{t \rightarrow \infty}{\longrightarrow} 0$ a.s. Next, we recall that by Lemma \ref{lemma:lemma_maxthm}, $C^{\star}(\theta)$ is non empty, compact and convex. Thus, applying Lemma \ref{lemma:tracking_c} yields immediately that $d_{\infty}\left(\left(N_x(t) / t\right)_{x \in \mathcal{X}}, C^{\star}(\theta^*)\right) \underset{t \rightarrow \infty}{\longrightarrow} 0$ a.s..
\end{proof}

\textbf{Theorem \ref{thm:thm_ubsce}}

Log Track-and-Stop satisfies the same sample complexity upper bound

$$ \mathbb{P}_{\theta}[\limsup _{\delta \rightarrow 0} \frac{\tau_{\delta}}{(\log(\frac{1}{\delta}))^2} \lesssim T^{\star}(\theta)] = 1$$

\begin{proof}

From Lemma \ref{lemma:lemma_contpsi} and Lemma \ref{lemma:lemma_maxthm} we know $\psi(\theta, w)$ is continuous in both $\theta$ and $w$ and $C^{\star}(\theta)$ is continuous in $\theta$. Note that

$$\mathcal{E} = \left\{ d_{\infty}((N_x(t)/t)_{x \in \mathcal{X}}, C^{\star}(\theta)) \rightarrow 0 \land \hat{\theta}_t \rightarrow  \theta \right\}$$

holds with probability 1 (Lemma 3, 5 and Proposition 1 in \citep{jedra2020optbailinear}). Let $\xi >0 $, By continuity of $\psi$, there exists an open neighborhood $\mathcal{V}(\xi)$ of $\left\{\theta\right\} \times C^{\star}(\theta)$ such that for all $(\theta', w') \in \mathcal{V}(\xi)$, it holds that

$$\psi(\theta', w') \geq (1-\xi)\psi(\theta, w^{\star})$$

for any $w^{\star} \in C^{\star}(\theta)$. Under $\mathcal{E}$, there exists $t_0 \geq 1$ such that for all $t \geq t_0$ it holds that $(\hat{\theta}_t, (N_x(t)/t)_{x \in \mathcal{X}}) \in \mathcal{V}(\xi)$, this for all $t \geq t_0$, it follows that

$$\psi(\hat{\theta}_t, (N_x(t)/t)_{x \in \mathcal{X}}) \geq (1-\xi)\psi(\theta, w^{\star})$$

By Lemma \ref{lemma:forced_expl}, there exists $t_1 \geq 1$ such that for all $t \geq t_1$ we have $\lambda_{\min}(\mathbf{A}_t) > c_{\mathcal{X}_0}\sqrt{t-d-t} > \kappa_0 \log (t)$, then $\lambda_{\min}(\mathbf{A}_t) \ge \kappa_0 \log (t)$ which implies $t \in B$.



We also have that under $\mathcal{E}$, there exists $t_2 \geq 1$ such that for all $t \geq t_0$ it holds that $\theta^{(1)}_t = \hat{\theta}_t$. Then, we can write
$$
\begin{aligned}
 Z(t) &= \inf _{\lambda \in \Alt(\theta^{(1)}_t)} \frac{1}{2}\|\theta^{(1)}_t-\lambda\|_{\mathbf{H}_t(\theta^{(1)}_t)}^2 \\ 
 &= t\inf _{\lambda \in \Alt(\theta^{(1)}_t)} \frac{1}{2}\|\theta^{(1)}_t-\lambda\|_{\mathbf{H}_{w_t}(\theta^{(1)}_t)}^2 \\ 
 &= t\psi(\hat{\theta}_t, (N_x(t)/t)_{x \in \mathcal{X}})
\end{aligned}
$$

Hence, under $\mathcal{E}$ and for $t \geq \max \{t_0, t_1, t_2 \}$

$$Z(t) \geq t (1-\xi)\psi(\theta, w^{\star})$$

Then

$$
\begin{aligned}
\tau_{\delta} & =\inf \left\{ t \geq 1 : Z(t) > \beta(\delta, t), t \in B \right\} \\
& \leq \max \{t_0, t_1, t_2 \} \lor \inf \left\{ t \geq 1 : t (1-\xi)\psi(\theta, w^{\star}) > \beta(\delta, t) \right\} \\
& \leq \max \{t_0, t_1, t_2 \} \lor \inf \left\{ t \geq 1 : t (1-\xi)\psi(\theta, w^{\star}) > \sqrt{2}(1+2S)\left( \frac{\sqrt{\log (t)}}{2} + \frac{4}{\sqrt{\log (t)}} \log \left(\frac{2^d}{\delta}\left(\frac{L t}{d}\right)^{\frac{d}{2}}\right)\right) \right\} \\
& \lesssim \max \left\{t_0, t_1, t_2,  \frac{1}{1-\xi}T^{\star}(\theta) \left(\log(\frac{1}{\delta})\right)^2\right\}
\end{aligned}
$$

Where the last inequality uses Leamma 8 in \citep{jedra2020optbailinear}. Thus

$$\mathbb{P}_{\theta}\left(\lim \sup _{\delta \rightarrow 0} \frac{\tau_{\delta}}{(\log(\frac{1}{\delta}))^2} \leqslant  T^{\star}(\theta) \right) = 1$$

\end{proof}

\textbf{Theorem \ref{thm:thm_ubsce_exp}}

Log Track-and-Stop satisfies the same sample complexity upper bound

$$\limsup _{\delta \rightarrow 0} \frac{\mathbb{E}_{\theta}[\tau]}{(\log \left(\frac{1}{\delta}\right))^2} \lesssim T^{\star}(\theta)$$

\begin{proof} Let $\varepsilon > 0$

\textbf{ Step 1.} By continuity of $\psi$ (see Lemma \ref{lemma:lemma_contpsi}), there exists $\xi_1(\varepsilon)>0$ such that for all $\theta^{\prime} \in \mathbb{R}^d$ and $w^{\prime} \in \Sigma$

\begin{align}\label{eq:cont_inequality}
    \left\{\begin{array}{ll}
\left\|\theta^{\prime}-\theta\right\| & \leq \xi_1(\varepsilon) \\
d_{\infty}\left(w^{\prime}, C^{\star}(\theta)\right) & \leq \xi_1(\varepsilon)
\end{array} \Longrightarrow\left|\psi\left(\theta, w^{\star}\right)-\psi\left(\theta^{\prime}, w^{\prime}\right)\right| \leq \varepsilon \psi\left(\theta, w^{\star}\right)=\varepsilon\left(T^{\star}(\theta)\right)^{-1}\right.
\end{align}

for any $w^{\star} \in \arg \min _{w \in C^{\star}(\theta)} d_{\infty}\left(w^{\prime}, w\right)$ (we have $w^{\star} \in C^{\star}(\theta)$ ). Furthermore, by the continuity properties of the correspondence $C^{\star}$ (see Lemma \ref{lemma:lemma_maxthm}), there exists $\xi_2(\varepsilon)>0$ such that for all $\theta^{\prime} \in \mathbb{R}^d$
$$
\left\|\theta-\theta^{\prime}\right\| \leq \xi_2(\varepsilon) \Longrightarrow \max _{w^{\prime \prime} \in C^{\star}\left(\theta^{\prime}\right)} d_{\infty}\left(w^{\prime \prime}, C^{\star}(\theta)\right)<\frac{\xi_1(\varepsilon)}{K-1}
$$

Additionally, let $\xi_3(\varepsilon) \le (S - \|\theta\|)/2$, then 

$$
\|\theta-\hat{\theta}_t\| \leq \xi_3(\varepsilon) \Longrightarrow \|\hat{\theta}_t\| \le S \Longrightarrow \hat{\theta}_t = \theta^{(1)}_t
$$

Let $\xi(\varepsilon)=\min \left(\xi_1(\varepsilon), \xi_2(\varepsilon), \xi_3(\varepsilon)\right)$. In the following, we construct $T_0$, and for each $T \geq T_0$ an event $\mathcal{E}_T$, under which for all $t \geq T$, it holds
$$
\|\theta-\hat{\theta}_t\| \leq \xi(\varepsilon) \Longrightarrow d_{\infty}\left(\left(N_x(t) / t\right)_{x \in \mathcal{X}}, C^{\star}(\theta)\right) \leq \xi_1(\varepsilon)
$$

Let $T \geq 1$, and define the following event
$$
\begin{aligned}
\mathcal{E}_{1, T} & =\bigcap_{t=T}^{\infty}\left\{\|\theta-\hat{\theta}_t\| \leq \xi(\varepsilon)\right\} 
\end{aligned}
$$

Note that, under the event $\mathcal{E}_{1, T}$, 
we have for all $t \geq T$


$$
\begin{aligned}
d_{\infty}\left(w(t), C^{\star}(\theta)\right) & \leq \max _{w^{\prime} \in C^{\star}\left(\hat{\theta}_t\right)} d_{\infty}\left(w^{\prime}, C^{\star}(\theta)\right) \\
& <\frac{\xi_1(\varepsilon)}{K-1}
\end{aligned}
$$

Define $\varepsilon_1=\xi_1(\varepsilon) /(K-1)$. By Lemma \ref{lemma:tracking_c} (6 of \citet{jedra2020optbailinear}) , there exists $t_1\left(\varepsilon_1\right) \geq T$ such that
$$
d_{\infty}\left(\left(N_a(t) / t\right)_{a \in \mathcal{A}}, C^{\star}(\mu)\right) \leq\left(p_t+d-1\right) \frac{\xi_1(\varepsilon)}{K-1} \leq \xi_1(\varepsilon)
$$
where $p_t=\left|\operatorname{supp}\left(\sum_{s=1}^t w(s)\right) \backslash \mathcal{X}_0\right|$ and more precisely $t_1\left(\varepsilon_1\right)=\max \left\{1 / \varepsilon_1^3, 1 /\left(\varepsilon_1^2 d\right), T / \varepsilon_1^3, 10 / \varepsilon_1\right\}$ (see the proof of Lemma 6 of \citet{jedra2020optbailinear}). Thus for $T \geq \max \left\{10 \varepsilon_1^2, \varepsilon_1 / d, 1\right\}$, we have $t_1\left(\varepsilon_1\right)=\left\lceil T / \varepsilon_1^3\right\rceil$. Hence, defining for all $T \geq \varepsilon_1^{-3}$, the event
$$
\mathcal{E}_T=\mathcal{E}_{1,\left\lceil\varepsilon_1^3 T\right\rceil} 
$$

we have shown that for all $T \geq T_0=\max \left(10 \varepsilon_1^5, \varepsilon_1^4 / d, \varepsilon_1^3, 1 / \varepsilon_1^3\right)$, the following holds

\begin{align} \label{eq:thm2_2}
\forall t \geq T, \quad\left\|\theta-\theta_t\right\| \leq \xi(\varepsilon) \Longrightarrow d_{\infty}\left(\left(N_x(t) / t\right)_{x \in \mathcal{X}}, C^{\star}(\theta)\right) \leq \xi_1(\varepsilon) .    
\end{align}

Finally, combining the implication Eq. \eqref{eq:thm2_2} with the fact that Eq. \eqref{eq:cont_inequality} holds under $\mathcal{E}_T$ we conclude that for all $T \geq T_0$, under $\varepsilon_T$ we have

\begin{align}\label{eq:thm_3} \psi\left(\hat{\theta}_t,\left(N_x(t) / t\right)_{x \in \mathcal{X}}\right) \geq(1-\varepsilon) \psi^{\star}(\theta)    
\end{align}

\textbf{Step 2.}: Let $T \geq T_0 \vee T_1 \vee T_2$ where $T_1$ is defined as
$$
T_1=\inf \left\{t \ge 1: t \in B\right\} .
$$

and $T_2$ is defined as


Under the event $\mathcal{E}_T$, for all $t \geq T$ we have

$$
Z(t)=t \psi\left(\hat{\theta}_t,\left(N_x(t) / t\right)_{x \in \mathcal{X}}\right) \text {, }
$$

Thus under the event $\mathcal{E}_T$, the inequality Eq. \eqref{eq:thm_3} holds, and for all $t \geq T$ we have
$$
Z(t)>t(1-\varepsilon)\left(T^{\star}(\theta)\right)^{-1}
$$

Under the event $\mathcal{E}_T$, we have
$$
\begin{aligned}
\tau & =\inf \left\{t \ge 1: Z(t)>\beta(\delta, t), t \in B\right\} \\
& \leq \inf \{t \geq T: Z(t)>\beta(\delta, t)\} \\
& \leq T \vee \inf \left\{ t \geq 1 : t (1-\varepsilon)\left(T^{\star}(\theta)\right)^{-1} > \sqrt{2}(1+2S)\left( \frac{\sqrt{\log (t)}}{2} + \frac{4}{\sqrt{\log (t)}} \log \left(\frac{2^d}{\delta}\left(\frac{L t}{d}\right)^{\frac{d}{2}}\right)\right) \right\} 
\end{aligned}
$$
Applying Lemma 8 in \citep{jedra2020optbailinear} yields

$$
\inf \left\{t \ge 1: t(1-\varepsilon)\left(T^{\star}(\theta)\right)^{-1} \geq \sqrt{2}(1+2S)\left( \frac{\sqrt{\log (t)}}{2} + \frac{4}{\sqrt{\log (t)}} \log \left(\frac{2^d}{\delta}\left(\frac{L t}{d}\right)^{\frac{d}{2}}\right)\right) \right\} \leq T_2^*(\delta) \text {, }
$$

where $T_2^{\star}(\delta)=\frac{c_1}{1-\varepsilon} T^{\star}(\theta) (\log (1 / \delta))^2+o(\log (1 / \delta))$ for some $0< c_1$ independent of $\delta$. This means for $T \geq \max \left\{T_0, T_1, T_2^{\star}(\delta)\right\}$, we have shown that

\begin{align}\label{eq:thm_4}
    \mathcal{E}_T \subseteq\{\tau \leq T\}    
\end{align}

Define $T_3^{\star}(\delta)=\max \left\{T_0, T_1, T_2^{\star}(\delta)\right\}$. We may then write for all $T \geq T_3^{\star}(\delta)$
$$
\tau \leq \tau \wedge T_3^{\star}(\delta)+\tau \vee T_3^{\star}(\delta) \leq T_3^{\star}(\delta)+\tau \vee T_3^{\star}(\delta) .
$$

Taking the expectation of the above inequality, and using the set inclusion Eq. \eqref{eq:thm_4}, we obtain that
$$
\mathbb{E}[\tau] \leq T_3^{\star}(\delta)+\mathbb{E}\left[\tau \vee T_3^{\star}(\delta)\right]
$$

Now we observe that
$$
\begin{aligned}
\mathbb{E}\left[\tau \vee T_3^{\star}(\delta)\right] & =\sum_{T=0}^{\infty} \mathbb{P}\left(\tau \vee T_3^{\star}(\delta)>T\right) \\
& =\sum_{T=T_3^{\star}(\delta)+1}^{\infty} \mathbb{P}\left(\tau \vee T_3^{\star}(\delta)>T\right) \\
& =\sum_{T=T_3^{\star}(\delta)+1}^{\infty} \mathbb{P}(\tau>T) \\
& \leq \sum_{T=T_3^{\star}(\delta)+1}^{\infty} \mathbb{P}\left(\mathcal{E}_T^{\mathsf{c}}\right) \\
& \leq \sum_{T=T_0 \vee T_1}^{\infty} \mathbb{P}\left(\mathcal{E}_T^{\mathsf{c}}\right)
\end{aligned}
$$

We have thus shown that

\begin{align} \label{eq:thm3_5}
    \mathbb{E}[\tau] \leq \frac{c_1}{1-\varepsilon} T^{\star}(\theta)(\log (1 / \delta))^2+\mathcal{O}(\log (1 / \delta))+T_0 \vee T_1+\sum_{T=T_{0} \vee T_1}^{\infty} \mathbb{P}\left(\mathcal{E}_T^{\mathsf{c}}\right) .    
\end{align}

\textbf{Step 3}: We now show that $\sum_{T=T_0 \vee T_1+1}^{\infty} \mathbb{P}\left(\mathcal{E}_T^c\right)<\infty$ and that it can be upper bounded by a constant independent of $\delta$. Let $T \geq T_0 \vee T_1$, we have
$$
\mathbb{P}\left(\mathcal{E}_T^{\mathsf{c}}\right) \leq \mathbb{P}\left(\mathcal{E}_{1,\left\lceil \varepsilon_1^3 T\right\rceil}^{\mathsf{c}}\right).
$$


We observe, using a union bound, Lemma 5 from \citet{jedra2020optbailinear} and Lemma \ref{lemma:tail_error_th}, that there exists strictly positive constants $c_3, c_4$ that are independent of $\varepsilon$ and $T$, and such that
$$
\begin{aligned}
\mathbb{P}\left(\mathcal{E}_{1,\left\lceil\varepsilon_1^3 T\right]}^{\mathrm{c}}\right) & \leq \sum_{t=\ell\left(\left[\epsilon_1^3 T\right]\right)}^{\infty} \mathbb{P}\left(\left\|\theta-\hat{\theta}_t\right\|>\xi(\varepsilon)\right) \\
& \leq \sum_{t=\ell\left(\left\lceil\varepsilon_1^3 T\right\rceil\right)}^{\infty} c_2 t^{d/2} \exp \left(- c_1 \xi(\varepsilon) t^{1/2}\right) 
\end{aligned}
$$

For $t$ large enough, the function $t \mapsto t^{d / 2} \exp \left(-c_1 \xi(\varepsilon) \sqrt{t}\right)$ becomes decreasing. Hence, for $T \geq T_2$, we have
$$
\mathbb{P}\left(\mathcal{E}_{2,\left\lceil\varepsilon_1^3 T\right\rceil}^c\right) \leq c_3 \int_{\left\lceil\varepsilon_1^3 T\right\rceil-1}^{\infty} t^{d / 2} \exp \left(-c_1 \xi(\varepsilon) \sqrt{t}\right) d t .
$$

Furthermore, for some $T_3 \geq T_2$ large enough, we may bound the integral for all $T \geq T_3$ as follows
$$
\int_{\ell\left(\left\lceil\varepsilon_1^3 T\right\rceil\right)-1}^{\infty} t^{d / 4} \exp \left(-c_1 \xi(\varepsilon)^2 \sqrt{t}\right) d t \lesssim \frac{\left(\left\lceil\varepsilon_1^3 T\right\rceil-1\right)^{d / 2+1}}{\xi(\varepsilon)^4 \exp \left(c_4 \xi(\varepsilon)^2 \sqrt{\left\lceil \varepsilon_1^3 T\right\rceil-1}\right)} .
$$

We spare the details of this derivation as the constants are irrelevant in our analysis. Essentially, the integral can be expressed through the upper incomplete Gamma function which can be upper bounded using some classical inequalities $[23,24]$. We then obtain that for $T \geq T_3$,
$$
\mathbb{P}\left(\mathcal{E}_{1, \left\lceil \varepsilon_1^3 T\right\rceil}^{\mathrm{c}}\right) \lesssim \frac{\left(\left\lceil\varepsilon_1^3 T\right\rceil-1\right)^{d / 2+1}}{\xi(\varepsilon)^4 \exp \left(c_4 \xi(\varepsilon)^2 \sqrt{\left\lceil \varepsilon_1^3 T\right\rceil-1}\right)} .
$$

Thus there exists $T_4 \geq T_3$ such that for all $T \geq T_4$,
$$
\mathbb{P}\left(\mathcal{E}_{2,\left\lceil\varepsilon_1^3 T\right\rceil}^c\right) \lesssim \frac{\ell\left(\left(\left\lceil\varepsilon_1^3 T\right\rceil\right)-1\right)^{d / 2+1}}{\xi(\varepsilon)^4 \exp \left(c_4 \xi(\varepsilon)^2 \sqrt{\ell\left(\left\lceil\varepsilon_1^3 T\right\rceil\right)-1}\right)} \lesssim \frac{T^{d / 2+1}}{\exp \left(c_5(\varepsilon) T^{\gamma / 2}\right)} .
$$

This shows that
$$
\begin{aligned}
\sum_{T=T_0 \vee T_1}^{\infty} \mathbb{P}\left(\mathcal{E}_{1,\left\lceil\varepsilon_1^3 T\right\rceil}^c\right) & =\sum_{T=T_0 \vee T_1}^{T_4} \mathbb{P}\left(\mathcal{E}_{1,\left\lceil\varepsilon_1^3 T\right\rceil}^c\right)+\sum_{T=T_4+1}^{\infty} \mathbb{P}\left(\mathcal{E}_{1,\left\lceil\varepsilon_1^3 T\right\rceil}^c\right) \\
& \lesssim \sum_{T=T_0 \vee T_1}^{T_4} \mathbb{P}\left(\mathcal{E}_{1,\left\lceil\varepsilon_1^3 T\right\rceil}^c\right)+\sum_{T=T_4+1}^{\infty} \frac{T^{d / 2+1}}{\exp \left(c_5(\varepsilon) T^{\gamma / 2}\right)} \\
& <\infty
\end{aligned}
$$
where the last inequality follows from the fact that we can upper bound the infinite sum by a Gamma function, which is convergent as long as $\gamma>0$.

Finally, we have thus shown that
$$
\sum_{T=T_0 \vee T_1+1}^{\infty} \mathbb{P}\left(\mathcal{E}_T^c\right)<\infty
$$

We note that this infinite sum depends $\varepsilon$ only.

\textbf{Last step:} Finally, we have shown that for all $\varepsilon>0$
$$
\mathbb{E}[\tau] \leq \frac{c_1}{1-\varepsilon} T^{\star}(\theta) (\log (1 / \delta))^2+\mathcal{O}(\log (1 / \delta))+T_0 \vee T_1+\sum_{T=T_0 \vee T_1}^{\infty} \mathbb{P}\left(\mathcal{E}_T^c\right)
$$
where $\sum_{T=T_0 \vee T_1}^{\infty} \mathbb{P}\left(\mathcal{E}_T^c\right)<\infty$ and is independent of $\delta$. Hence,
$$
\limsup _{\delta \rightarrow 0} \frac{\mathbb{E}_{\theta}\left[\tau_\delta\right]}{(\log (1 / \delta))^2} \lesssim \frac{c_1}{1-\varepsilon} T^{\star}(\theta)
$$

Letting $\varepsilon$ tend to 0, we conclude that
$$
\limsup _{\delta \rightarrow 0} \frac{\mathbb{E}_{\theta}\left[\tau_\delta\right]}{(\log (1 / \delta))^2} \lesssim T^{\star}(\theta)
$$

\end{proof}

\section{Examples}\label{section:examples}

We present three concrete examples of pure exploration problems under logistic bandits. First we need an auxiliary Lemma.

\begin{lemma} (Lemma 5 \citep{degenne2020pureexp})
For $\theta, \lambda \in \mathbb{R}^d$, $w$ in the interior of the probability simplex $\Sigma, y \in \mathbb{R}^d, x \in \mathbb{R}$, we have

    $$
    \inf _{\lambda:\lambda^{\top} y \geq x} \frac{\|\theta-\lambda\|_{\mathbf{H}_w(\theta)}^2}{2}=\left\{\begin{array}{ll}
    \frac{(x-\theta^{\top} y)^2}{2\|y\|_{\mathbf{H}_w(\theta)^{-1}}^2} & \text { if } x \geq\theta^{\top} y \\
    0 & \text { otherwise }
    \end{array} .\right.
    $$

And 

$$
    \inf _{\lambda:\lambda^{\top} y \leq x} \frac{\|\theta-\lambda\|_{\mathbf{H}_w(\theta)}^2}{2}=\left\{\begin{array}{ll}
    \frac{(x-\theta^{\top} y)^2}{2\|y\|_{\mathbf{H}_w(\theta)^{-1}}^2} & \text { if } x \leq \theta^{\top} y \\
    0 & \text { otherwise }
    \end{array} .\right.
    $$
    
\end{lemma}

\begin{proof}
We consider the Lagrangian of the problem, and we obtain

    $$
    \begin{aligned}
    \inf _{\lambda:\lambda^{\top} y\geq x} \frac{\|\theta-\lambda\|_{\mathbf{H}_w(\theta)}^2}{2} & =\sup _{\alpha \geq 0} \inf _{\lambda \in \mathbb{R}^d} \frac{\|\theta-\lambda\|_{\mathbf{H}_w(\theta)}^2}{2}+\alpha(x-\lambda^{\top} y) \\
    & =\sup _{\alpha \geq 0} \alpha(x-\theta^{\top} y)-\alpha^2 \frac{\|y\|_{\mathbf{H}_w(\theta)}^2}{2} \\
    & = \begin{cases}\frac{(x-\theta^{\top} y)^2}{2\|y\|_{\mathbf{H}_w(\theta)^{-1}}^2} & \text { if } x \geq\theta^{\top} y \\
    0 & \text { otherwise }\end{cases}
    \end{aligned}
    $$
    
where the infimum in the first equality is reached at $\lambda=\theta+\alpha \mathbf{H}_w^{-1} (\theta) y$ and the supremum in the last equality is reached at $\alpha=(x-\langle\theta, y\rangle) /\|y\|_{\mathbf{H}_w^{-1} (\theta)}^2$ if $x \geq\langle\theta, y\rangle$ and at $\alpha=0$ else. The second equality can be solved with the same steps.
\end{proof}

\subsection{Best arm identification}

In this example $i^{\star}(\theta) = \arg \max _{x \in \mathcal{X}} \{\mu(x^{\top} \theta) \} = \arg \max _{x \in \mathcal{X}} \{x^{\top} \theta \}$. 
Lets define $x^{\star} (\theta):=i^{\star}(\theta)$.

\textbf{Lemma \ref{lemma:example_bai}}
For all $\theta \in \mathbb{R}^d$ such that $i^{\star}(\theta)$ is unique,

$$
T^{\star}(\theta)^{-1}=\max _{w \in \Sigma} \min _{x \neq x^{\star} (\theta)} \frac{\left(\theta^{\top}x^{\star} (\theta)-\theta^{\top}x\right)^2}{2\left\|x^{\star} (\theta)-x\right\|_{\mathbf{H}_w^{-1}}^2}
$$

\begin{proof}
Recall that the characteristic time is given by
    $$    
    \begin{aligned}
    T^{\star}(\theta)^{-1}=\max _{w \in \Sigma} \inf _{\lambda \in \Alt(\theta)} \frac{\|\theta-\lambda\|_{\mathbf{H}_w}^2}{2}
    \end{aligned}
    $$

$$
    \begin{aligned}
T^{\star}(\theta)^{-1} & =\max _{w \in \Sigma} \min _{x \neq x^{\star}(\theta)} \inf _{\lambda:\lambda^{\top} x>\lambda^{\top} x^{\star}(\theta)} \frac{\|\theta-\lambda\|_{\mathbf{H}_w}^2}{2} \\
& =\max _{w \in \Sigma} \min _{x \neq x^{\star}(\theta)} \frac{\left(\theta^{\top}x^{\star} (\theta)-\theta^{\top}x\right)^2}{2\left\|x^{\star}(\theta)-a\right\|_{\mathbf{H}_w^{-1}}^2}
\end{aligned}
$$

\end{proof}

\subsection{Thresholding bandit problem}

In this example, given $\rho \in (0, 1)$, $i^{\star}(\theta) = \{x \in \mathcal{X}: \mu(x^{\top} \theta) > \rho\} = \{x \in \mathcal{X}: x^{\top} \theta > \mu^{-1}(\rho)\} $. 

\textbf{Lemma \ref{lemma:example_tbp}}
    For all $\theta \in \mathbb{R}^d$ such that $i^{\star}(\theta)$ is unique,

$$
T^{\star}(\theta)^{-1}=\max _{w \in \Sigma} \min _{x \in \mathcal{X}} \frac{(\theta^{\top} x - \mu^{-1}(\rho))^2}{2\left\|x\right\|_{\mathbf{H}_w^{-1}}^2}
$$

\begin{proof}
    $$
    \begin{aligned}
    T^{\star}(\theta)^{-1} & =\max _{w \in \Sigma} \inf _{\lambda \in \Alt (\theta)} \frac{\|\theta-\lambda\|_{\mathbf{H}_w}^2}{2} \\
    & =\max _{w \in \Sigma} \min \left( \min _{x: \mu(\theta^{\top}x)<\rho} \inf _{\lambda: \mu(x^{\top} \lambda) > \rho} \frac{\|\theta-\lambda\|_{\mathbf{H}_w}^2}{2},  \min _{x: \mu(\theta^{\top}x)>\rho}\inf _{\lambda: \mu(x^{\top} \lambda) < \rho} \frac{\|\theta-\lambda\|_{\mathbf{H}_w}^2}{2} \right)\\
    & =\max _{w \in \Sigma} \min \left( \min _{x: \mu(\theta^{\top}x)<\rho} \inf _{\lambda: x^{\top} \lambda > \mu^{-1}(\rho)} \frac{\|\theta-\lambda\|_{\mathbf{H}_w}^2}{2},  \min _{x: \mu(\theta^{\top}x)>\rho}\inf _{\lambda: x^{\top} \lambda < \mu^{-1}(\rho)} \frac{\|\theta-\lambda\|_{\mathbf{H}_w}^2}{2} \right)\\
    & =\max _{w \in \Sigma} \min _{x \in \mathcal{X}} \frac{(\theta^{\top} x - \mu^{-1}(\rho))^2}{2\|x\|_{\mathbf{H}_w^{-1}}^2}
    \end{aligned}
    $$
\end{proof}

\subsection{Best-M arm identification}

In this example, given $M \in [K-1]$, $i^{\star}(\theta) = \{x_{(i)}(\theta): i\in [M]\}$  where $x_{(1)}(\theta), \dots, x_{(K)}(\theta)$ are the ordered arms with respect to their expected rewards. 


\begin{lemma}
    For all $\theta \in \mathbb{R}^d$ such that $x_{(1)}(\theta)^{\top} \theta < \dots < x_{(K)}(\theta)^{\top} \theta$,

$$
T^{\star}(\theta)^{-1}=\max _{w \in \Sigma} \min _{x \ne x_{(M)}(\theta)} \frac{(\theta^{\top} x - \theta^{\top}x_{(M)}(\theta))^2}{2\left\|x\right\|_{\mathbf{H}_w^{-1}}^2}
$$

\end{lemma}

\begin{proof}
    $$
    \begin{aligned}
    T^{\star}(\theta)^{-1} & =\max _{w \in \Sigma} \inf _{\lambda \in \Alt (\theta)} \frac{\|\theta-\lambda\|_{\mathbf{H}_w}^2}{2} \\
    & =\max _{w \in \Sigma} \min \left( \min _{x: x \notin i^{\star}(\theta)} \inf _{\lambda: \mu(x^{\top} \lambda) \ge \mu(x_{(M)}(\theta)^{\top}\lambda) } \frac{\|\theta-\lambda\|_{\mathbf{H}_w}^2}{2},  \min _{x: x \in i^{\star}(\theta)} \inf _{\lambda: \mu(x^{\top} \lambda) < \mu(x_{(M)}(\theta)^{\top}\lambda) } \frac{\|\theta-\lambda\|_{\mathbf{H}_w}^2}{2} \right)\\
    & =\max _{w \in \Sigma} \min \left( \min _{x: x \notin i^{\star}(\theta)} \inf _{\lambda: x^{\top} \lambda \ge x_{(M)}(\theta)^{\top}\lambda } \frac{\|\theta-\lambda\|_{\mathbf{H}_w}^2}{2},  \min _{x: x \in i^{\star}(\theta)} \inf _{\lambda: x^{\top} \lambda < x_{(M)}(\theta)^{\top}\lambda } \frac{\|\theta-\lambda\|_{\mathbf{H}_w}^2}{2} \right)\\
    & =\max _{w \in \Sigma} \min _{x \ne x_{(M)}(\theta)} \frac{(\theta^{\top} x - \theta^{\top} x_{(M)}(\theta))^2}{2\|x\|_{\mathbf{H}_w^{-1}}^2}
    \end{aligned}
    $$
\end{proof}

\section{Useful lemmas}

\begin{lemma}\label{lemma:det-trace-inneq}
    (Determinant-Trace inequality). Let $\left\{x_s\right\}_{s=1}^{\infty}$ a sequence in $\mathbb{R}^d$ such that $\left\|x_s\right\|_2 \leq X$ for all $s \in \mathbb{N}$, and let $\lambda$ be a non-negative scalar. For $t \geq 1$ define $\mathbf{A}_t:=\sum_{s=1}^{t} x_s x_s^{\top}$. The following inequality holds:
$$
\operatorname{det}\left(\mathbf{A}_{t}\right) \leq\left(t X^2 / d\right)^d
$$
\end{lemma}

\end{alphasection}

\end{document}